%% file: main.tex
\documentclass{article} % DO NOT CHANGE THIS
\usepackage[english]{babel}
\usepackage[letterpaper,top=2cm,bottom=2cm,left=3cm,right=3cm,marginparwidth=1.75cm]{geometry}
 \usepackage{etex}
\usepackage[utf8]{inputenc}
\usepackage[pdftex]{graphicx}

\usepackage[T1]{fontenc}    % use 8-bit T1 fonts
\usepackage{xcolor}         % colors
\usepackage{hyperref} 
\usepackage{times}  % DO NOT CHANGE THIS
\usepackage{helvet}  % DO NOT CHANGE THIS
\usepackage{courier}  % DO NOT CHANGE THIS
\usepackage{graphicx} % DO NOT CHANGE THIS
\usepackage[square,sort,comma,numbers]{natbib}
\usepackage{caption} % DO NOT CHANGE THIS AND DO NOT ADD ANY OPTIONS TO IT
\usepackage{algorithm}
\usepackage{algorithmic}
\usepackage{amsmath}
\usepackage{bm}
\usepackage{bbm}
\usepackage{amsfonts,amsthm}
\usepackage{subcaption}
\usepackage{newfloat}
\usepackage{listings}
\newcommand{\ccalL}{\mathcal{L}}
\newcommand{\ccalA}{\mathcal{A}}
\newcommand{\ccalH}{\mathcal{H}}
\newcommand{\ccalM}{\mathcal{M}}
\newcommand{\reals}{{\mathbb{R}}}
\newcommand{\bbh}{{\mathbf{h}}}
\newcommand{\bbx}{{\mathbf{x}}}
\newcommand{\bby}{{\mathbf{y}}}
\newcommand{\bbH}{{\mathbf{H}}}
\newcommand{\bbG}{{\mathbf{G}}}
\newcommand{\bbL}{{\mathbf{L}}}
\newcommand{\bbV}{{\mathbf{V}}}
\newcommand{\bbW}{{\mathbf{W}}}
\newcommand{\bbE}{{\mathbf{E}}}
\newcommand{\ccalO}{\mathcal{O}}

\newtheorem*{theorem*}{Theorem}
\newtheorem{definition}{\hspace{0pt}\bf Definition}
\newtheorem{assumption}{\hspace{0pt}\bf Assumption}
\newtheorem{proposition}{\hspace{0pt}\bf Proposition}
\newtheorem{lemma}{\hspace{0pt}\bf Lemma}
\newtheorem{theorem}{\hspace{0pt}\bf Theorem}

\title{Generalization of Graph Neural Networks\\ is Robust to Model Mismatch}
\author{
    Zhiyang Wang$^*$, Juan Cervi\~no$^\dagger$, Alejandro Ribeiro$^*$\\\\
    *University of Pennsylvania, USA \qquad $\dagger$ Massachusetts Institute of Technology, USA \\
 zhiyangw@seas.upenn.edu, jcervino@mit.edu, aribeiro@seas.upenn.edu
}

\begin{document}
\date{}
\maketitle
\begin{abstract}
Graph neural networks (GNNs) have demonstrated their effectiveness in various tasks supported by their generalization capabilities. However, the current analysis of GNN generalization relies on the assumption that training and testing data are independent and identically distributed (i.i.d). This imposes limitations on the cases where a model mismatch exists when generating testing data. In this paper, we examine GNNs that operate on geometric graphs generated from manifold models, explicitly focusing on scenarios where there is a mismatch between manifold models generating training and testing data. Our analysis reveals the robustness of the GNN generalization in the presence of such model mismatch. This indicates that GNNs trained on graphs generated from a manifold can still generalize well to unseen nodes and graphs generated from a mismatched manifold. We attribute this mismatch to both node feature perturbations and edge perturbations within the generated graph. Our findings indicate that the generalization gap decreases as the number of nodes grows in the training graph while increasing with larger manifold dimension as well as larger mismatch. Importantly, we observe a trade-off between the generalization of GNNs and the capability to discriminate high-frequency components when facing a model mismatch. The most important practical consequence of this analysis is to shed light on the filter design of generalizable GNNs robust to model mismatch. We verify our theoretical findings with experiments on multiple real-world datasets.

\end{abstract}

\section{Introduction}
Graph Neural Networks (GNNs) \cite{sandryhaila2013discrete,kipf2017semi, gama2019convolutional}, as a deep learning model on graphs, have been unarguably one of the most recognizable architectures when processing graph-structured data. GNNs have achieved notable performances in numerous applications, such as recommendation systems \cite{wu2022graph}, protein structure predictions \cite{yin2023coco}, and multi-agent robotic control \cite{gosrich2022coverage}. These outstanding results of GNNs depend on their 
empirical performances when \textit{predicting} over unseen testing data. This is evaluated in theory with \textit{statistical generalization analysis}, which quantifies the difference between the \textit{empirical risk} (i.e. training error) and the \textit{statistical risk }(i.e. testing error) in deep learning theory \cite{kawaguchi2022generalization}. 
\begin{figure*}[t]
    % \captionsetup{skip=0pt} % Adjust the skip value as needed
    \centering
    \begin{subfigure}[b]{0.3\textwidth}
        \centering
        \includegraphics[trim=30 30 30 30, clip, width=\linewidth]{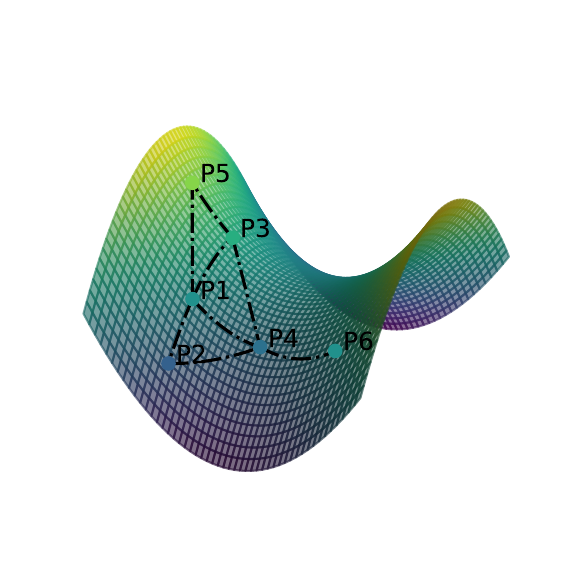}
        \caption{Manifold $\ccalM$}
        \label{fig:manifold_clean}
    \end{subfigure}
    \hfill
    \begin{subfigure}[b]{0.3\textwidth}
        \centering
        \includegraphics[trim=30 30 30 30, clip, width=\linewidth]{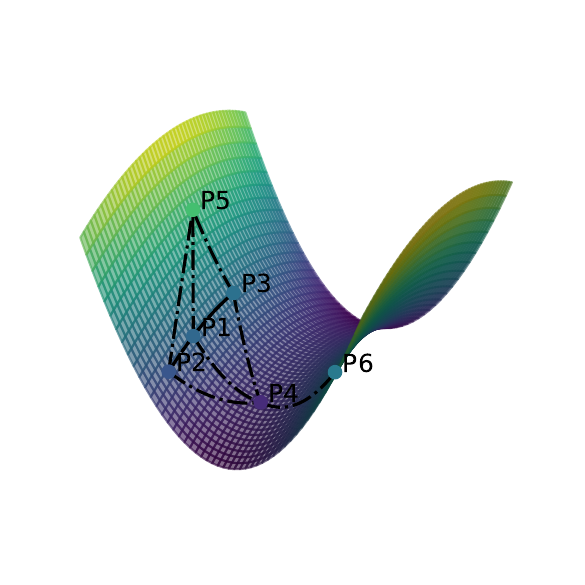}
        \caption{Mismatched Manifold $\ccalM'$}
        \label{fig:manifold_perturbed}
    \end{subfigure}
        \hfill
    \begin{subfigure}[b]{0.3\textwidth}
        \centering
        \includegraphics[trim=30 30 30 30,  clip, width=\linewidth]{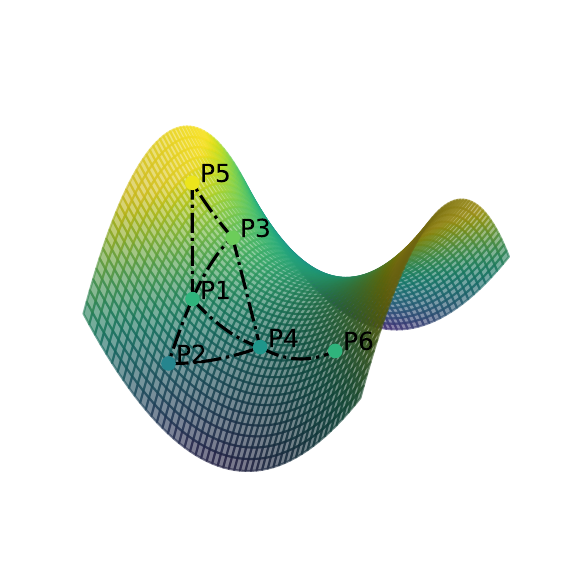}
        \caption{Perturbed manifold function $f'$}
        \label{fig:manifold_node}
    \end{subfigure}
    \caption{Example of model mismatch. (a) The original manifold with a generated graph based on sampled points ($P1,\cdots, P6$). (b) The mismatched manifold with the sampled points also shifted, resulting in a perturbed graph. (c) The manifold mismatch can be seen as the perturbation of manifold function values, which leads to perturbed node features on the generated graph. 
    }
    \label{fig:manifold}
\end{figure*}
Recent works have focused on proving the generalization bounds of GNNs without any dependence on the underlying model responsible for generating the graph data \cite{scarselli2018vapnik, garg2020generalization, verma2019stability}. Generalization analysis on graph classification is studied in a series of works when graphs are drawn from random limit models \cite{ruiz2023transferability, maskey2022generalization, maskey2024generalization, levie2024graphon}. In \cite{wang2024manifold}, the authors study the generalization of GNNs over graphs generated from an underlying manifold on both node and graph levels. These works assume that the training and testing graphs are generated from the same underlying model. In practice, there are inevitable scenarios with generative model mismatch between testing and training graphs \cite{li2022ood}.  Hence, it is of crucial to demonstrate the generalization ability of GNNs remains robust to generative model mismatch. This would provide a promising assurance that GNNs can maintain outstanding generalizable performance even in noisy environments.

The model mismatch may stem from disturbances during the graph generation process from the manifold. Moreover, the underlying manifold model is prone to undergo alternations and fluctuations in practical situations. While we take into account the underlying model mismatches, they can be interpreted as perturbations within the generated graph domain. Figure 1a displays the original manifold and a graph derived from it. Figure \ref{fig:manifold_perturbed} illustrates a mismatched manifold, which leads to edge perturbations in the generated graph. Figure \ref{fig:manifold_node} shows the interpretation of perturbed manifold function values, resulting in node feature perturbations in the generated graph.

We prove that the generalization of GNNs is robust to model mismatches based on the convergence of GNNs on generated graphs to neural networks on the manifold model \cite{wang2023geometric}, combined with the stability of the Manifold Neural Networks (MNNs) under manifold deformations \cite{wang2024stability}. Implementing low-pass and integral Lipschitz continuous filters (Definition \ref{def:low-pass}) allows us to bound the generalization gap under model mismatches. This bound decreases with the number of nodes in the training graph and increases with both the mismatch size and the manifold dimension. The key insight from the bound is that a more robust generalization necessitates the cost of failing to discriminate high spectral components in the graph data. 

Our main contributions are as follows:
\begin{enumerate}
    \item We analyze the generalization of the GNNs when there exists a manifold model mismatch between the training data and testing data.
    \item  We determine the manifold mismatch as perturbations on the generated graphs, both as node feature perturbations and edge perturbations.
    \item We propose that implementing continuity restrictions on the filters composing the GNN ensures the robust generalization in both node and graph classification tasks.
    \item We observe a trade-off between robust generalization and discriminability through the generalization bound.
\end{enumerate}
We conduct experiments on real-world datasets to validate our theoretical findings.

\section{Related Works}

\subsection{In-Distribution Generalization of GNNs}
Existing works on the in-distribution generalization of GNNs fall in node and graph level tasks. For node classification tasks of GNNs, there are works providing a generalization bound of GNNs based on a Vapnik-Chervonenkis dimension \cite{scarselli2018vapnik}, algorithmic stability analysis \cite{verma2019stability, zhou2021generalization}, PAC-Bayesian analysis \cite{ma2021subgroup} and Rademacher complexity \cite{esser2021learning}. For graph classification tasks of GNNs, the authors prove the generalization bound via Rademacher complexity \cite{garg2020generalization} and PAC-Bayes analysis \cite{liao2020pac, ju2023generalization}. The authors consider a continuous graph limit model to analyze the generalization of GNNs on graph classification
in \cite{maskey2022generalization, maskey2024generalization,levie2024graphon}. In \cite{wangcervino2024}, the authors prove the generalization of GNNs on graphs sampled from a manifold both for node and graph classification tasks. These works are considered only in the in-distribution case where the training and testing data are sampled from the same distribution.

\subsection{Out-of-Distribution Generalization of GNNs}
Several works have extensively addressed the out-of-distribution generalization of GNNs with graph enhancement methods. 
The authors in \cite{tian2024graphs} propose a domain generalization framework for node-level tasks on graphs to address distribution shifts in node attribute distribution and graphic topology. In \cite{fan2023generalizing}, the authors study the out-of-distribution generalization of GNNs on graph-level tasks with a causal representation learning framework. In \cite{li2022ood} the authors handle graph distribution shifts in complex
and heterogeneous situations of GNNs with a nonlinear graph representation decorrelation method. The authors in \cite{yehudai2021local} propose a size generalization analysis of GNNs correlated to the discrepancy between local distributions of graphs. In our novel approach, we conceptualize the generative model mismatch as a distribution shift. Our findings can be further generalized into a theoretical framework for GNNs in manifold domain shift scenarios. In such cases, the generalization gap is directly proportional to the distance between the manifold models. This extension is discussed in more detail within the supplementary material.
\subsection{GNNs and Manifold Neural Networks}
Geometric deep learning has been proposed in \cite{bronstein2017geometric} with neural network architectures on manifolds. The authors in \cite{monti2017geometric} and \cite{chakraborty2020manifoldnet} provide neural network architectures for manifold-valued data.  In \cite{wang2024stability} and \cite{wang2022convolutional}, the authors define convolutional operation over manifolds and see the manifold convolution as a generalization of graph convolution, which establishes the limit of neural networks on large-scale graphs as manifold neural networks (MNNs). The authors in \cite{wang2023geometric} further establish the relationship between GNNs and MNNs with non-asymptotic convergence results for different graph constructions.

\section{Neural Networks on Manifolds}
Suppose there is a $d$-dimensional embedded Riemannian manifold $\ccalM\subset \reals^\mathsf{M}$ which is compact, smooth and differentiable. 
A measure $\mu$ over $\ccalM$ with density function $\rho:\ccalM\rightarrow(0,\infty)$, which is assumed to be bounded as $0<\rho_{min}\leq \rho(x) \leq \rho_{max}<\infty$ for each $x\in \ccalM$. Data supported over the manifold is defined as a scalar function $f:\ccalM \rightarrow \reals$ \cite{wang2024stability} mapping datum value $f(x)$ to each point $x\in \ccalM$. The manifold with density $\rho$ is endowed with a weighted Laplace operator \cite{grigor2006heat}, which generalizes the Laplace-Beltrami operator as
\begin{equation}
    \label{eqn:weight-Laplace}
    \ccalL f = -\frac{1}{2\rho} \text{div}(\rho^2 \nabla f),
\end{equation}
where $\text{div}$ denotes the divergence operator of $\ccalM$ and $\nabla$ denotes the gradient operator of $\ccalM$ \cite{bronstein2017geometric}. 

We consider square-integrable functions over $\ccalM$, denoted as use $L^2(\ccalM)$. The inner product of functions $f, g\in L^2(\ccalM)$ is defined as 
\begin{equation}
\label{eqn:innerproduct}
    \langle f,g \rangle_{\ccalM}=\int_{\ccalM} f(x)g(x) \text{d}\mu(x), 
\end{equation}
while the $L^2$ norm is defined as $\|f\|^2_{\ccalM } = \langle f, f\rangle_{\ccalM }$.

Manifold convolution is defined by aggregating the heat diffusion process over $\ccalM $ with operator $\ccalL $ \cite{wang2022convolutional, wang2024stability}. With the input manifold function $f\in L^2(\ccalM )$, the manifold convolution output is 
\begin{align} 
\label{eqn:manifold-convolution}
   g(x) = \bbh(\ccalL )f(x) =\sum_{k=0}^{K-1} h_k e^{-k\ccalL }f(x) \text{.}
\end{align}

Considering the frequency representation, the Laplace operator $\ccalL $ has real, positive and discrete eigenvalues $\{\lambda_i\}_{i=1}^\infty$ as the Laplace operator is self-adjoint and positive-semidefinite. The eigenfunction associated with each eigenvalue is denoted as $\bm\phi_i$, i.e. $\ccalL  \bm\phi_i =\lambda_i \bm\phi_i$. The eigenvalues are ordered as $0=\lambda_1\leq \lambda_2\leq \lambda_3\leq \hdots$, and the eigenfunctions are orthonormal and form the eigenbasis of $L^2(\ccalM )$. When mapping a manifold function onto one of the eigenbasis $\bm\phi_i$, we have the spectral component as $[\hat{f} ]_i=\langle f, \bm\phi_i\rangle_{\ccalM }$. The manifold convolution can be written in the spectral domain point-wisely as 
\begin{align}
     [\hat{g} ]_i = \sum_{k=0}^{K-1} h_k e^{-k\lambda_i}   [\hat{f} ]_i\text{.}
\end{align}
Hence, the frequency response of manifold filter is given by $\hat{h}(\lambda)=\sum_{k=0}^{K-1} h_k e^{-k\lambda}$, depending only on the filter coefficients $h_k$ and eigenvalues of $\ccalL $ when $\lambda=\lambda_i$.

Manifold neural networks (MNNs) are built by cascading layers consisting of a bank of manifold filters and a pointwise nonlinearity function $\sigma:\reals \rightarrow\reals$, with the output of each layer written as 
\begin{equation}\label{eqn:mnn}
f_l(x) = \sigma\Big( \bbh_l (\ccalL ) f_{l-1}(x)\Big).
\end{equation}
To represent the MNN succinctly, we group all learnable parameters, and denote the mapping based on input manifold function $f\in L^2(\ccalM )$ to predict target function $g \in L^2(\ccalM )$ as $\bm\Phi(\bbH,\ccalL ,f)$, where $\bbH\in\ccalH\subset\reals^P$  is a filter parameter set of the manifold filters. A positive loss function is denoted as $\ell(\bm\Phi(\bbH, \ccalL, f), g)$ to measure the estimation performance.

\section{Geometric Graph Neural Networks}
Suppose we can access a discrete set of sampled points over manifold $\ccalM$. A graph $\bbG$ is generated based on a set of $N$ i.i.d. randomly sampled points $X_N=\{x_N^1, x_N^2,\cdots, x_N^{N}\}$ according to measure $\mu$ over $\ccalM$. Seeing these $N$ sampled points as nodes, edges connect every pair of nodes $(x_N^i, x_N^j)$ is connected with weight value $[\bbW_N]_{ij}, \bbW_N\in \reals^{N \times N }$ determined by a function of their Euclidean distance $\|x_N^i - x_N^j\|$ \cite{calder2022improved}, explicitly written as 
\begin{equation}
\label{eqn:edge-weight}
[\bbW_N]_{ij} = \frac{\alpha_d}{(d+2)N \epsilon^{d+2} } \mathbbm{1}_{[0,1]}\left( \frac{\|x_N^i - x_N^j\|}{\epsilon} \right),
\end{equation}
where $\alpha_d$ is the volume of the $d$-dimensional Euclidean unit ball and $\mathbbm{1}$ represents an indicator function. Based on this, the graph Laplacian can be calculated as $\bbL_N = \text{diag}(\bbW_N \mathbf{1}) - \bbW_N$. On this generated graph, graph data values are sampled from the functions over manifold $\ccalM$ \cite{wang2022convolutional, chew2023convergence}. Consider the input and target functions $f,g\in L^2(\ccalM)$, the sampled input and target functions $\bbx, \bby \in L^2(X_N)$ over graph $\bbG$  can be written  as 
\begin{align}
\label{eqn:sample}
    [\bbx]_i = f(x_N^i),\quad [\bby ]_i = g(x_N^i).
\end{align}

A convolutional filter on graph $\bbG$ can be extended from manifold convolution defined in \eqref{eqn:manifold-convolution} by replacing the Laplace operator with the graph Laplacian as 
\begin{equation}
    \label{eqn:graph-convolution}
    \bby = \bbh(\bbL )\bbx = \sum_{k=0}^{K-1} h_k e^{-k\bbL} \bbx.
\end{equation}
We observe that this formation is accordant with the definition of the graph convolution \cite{gama2019convolutional} with graph shift operator as $e^{-\bbL}$.
Replace $\bbL$ with eigendecomposition $\bbL = \bbV \bm\Lambda \bbV^H$, where $\bbV$ is the eigenvector matrix and $\bm\Lambda$ is a diagonal matrix with eigenvalues of $\bbL$ as the entries. The spectral representation of this filter on $\bbG$ is 
\begin{equation}
    \label{eqn:graph_convolution_spectral}
    \bbV^H \bbh(\bbL_\tau) \bbx =  \sum_{k=1}^{K-1} h_k e^{-\Lambda k} \bbV^H \bbx = \hat{h}(\bm\Lambda)\bbV^H \bbx.
\end{equation}
This analogously leads to a frequency response of this graph convolution, i.e.   $\hat{h}(\lambda)= \sum_{k=0}^{K-1} h_k \lambda^k$, relating the input and output spectral components point-wisely.

Neural networks on $\bbG$ is composed of layers consisting of graph filters and point-wise nonlinearity $\sigma$, written as 
\begin{equation}
    \label{eqn:gnn}
    \bbx_l = \sigma\Big(\bbh_l(\bbL)\bbx_{l-1}\Big).
\end{equation}
The mapping from input graph data $\bbx\in L^2(X_N)$ to predict $\bby \in L^2(X_N)$ is denoted as $\bm\Phi(\bbH,\bbL,\bbx)$ with $\bbH\in \ccalH\subset{\reals^P}$ which is trained to minimize the loss $\ell(\bm\Phi(\bbH,\bbL,\bbx), \bby)$.

\section{Generalization of GNNs to model mismatch}

\label{sec:generalization}
\subsection{Manifold Model Mismatch}
A mismatched model of manifold $\ccalM$ is denoted as $\ccalM^\tau$ where $\tau$ maps each point $x\in\ccalM$ to a displaced $\tau(x)\in \ccalM^\tau$. We restrict this mismatch function class $\tau$ as it preserves the properties of $\ccalM$ and denote the curvature distance between $x$ and the displaced $\tau(x)$ as $\text{dist}(x,\tau(x))$. The mismatch $\tau$ induces a tangent map $\tau_{*,x}: T_x\ccalM\rightarrow T_{\tau(x)}\ccalM$ {which is a linear map between the tangent spaces \cite{loring2011introduction}.} With the coordinate description over $\ccalM$, the tangent map $\tau_{*,x}$ can be exactly represented by the Jacobian matrix $J_x(\tau)$. 

The manifold model mismatch can be seen as deformations to input manifold functions as shown in Figure \ref{fig:manifold_node}.
\begin{equation}
\label{eqn:perturb-node}
    \ccalL f(\tau(x)) = \ccalL f'(x), \qquad x\in \ccalM.
\end{equation}
While the model mismatch can also be understood as deformations to the Laplacian operator as shown in Figure \ref{fig:manifold_perturbed}.
\begin{equation}
\label{eqn:perturb-edge}
    \ccalL f(\tau(x)) = \ccalL_\tau f(x), \qquad x\in \ccalM.
\end{equation}

\subsection{Generalization of GNNs on node-level}
The generalization analysis is restricted to a finite-dimensional subset of $L^2(\ccalM)$, i.e. the functions over the manifold are bandlimited as defined in Definition \ref{def:band}.
\begin{definition}
\label{def:band}
      Function $f\in L^2(\ccalM)$ is bandlimited if there exists some $\lambda>0$ such that for all eigenpairs $\{\lambda_i, \bm\phi_i\}_{i=1}^\infty$ of the weighted Laplacian $\ccalL $ when $\lambda_i>\lambda$, we have $\langle f, \bm\phi_i \rangle_{\ccalM} = 0$.
\end{definition}

\begin{assumption}(Lipschitz target function)\label{ass:loss}
 The target function $g$ is Lipschitz continuous, i.e., $|g(x)-g(y)|\leq C_g\text{dist}(x,y)$, for all  $x,y\in\ccalM$.
\end{assumption}

We further assume that the filters in MNN $\bm\Phi(\bbH,\ccalL , \cdot)$ and GNN $\bm\Phi(\bbH,\bbL ,\cdot)$ are low-pass and integral Lipschitz filters as defined in Definition \ref{def:low-pass}. We note that this is a mild assumption as high frequency components on the graph/manifold implies that there exist functions with large variations in adjacent entries. This naturally generates instabilities that are more difficult to learn.

\begin{definition}\label{def:low-pass}
A filter is a low-pass and integral Lipschitz filter if its frequency response satisfies 
\begin{align}
    \label{eqn:low-pass}
       & \left|\hat{h}(\lambda)\right| =\ccalO\left(\lambda^{-d}\right),\quad \lambda \rightarrow \infty,\\
       & \left|\hat{h}'(\lambda)\right| \leq C_L \lambda^{-d-1},\quad \lambda\in(0,\infty).
\end{align}
    with $d$ denoted as the dimension of manifold $\ccalM$ and $C_L$ a positive continuity constant.
\end{definition}
We note that with a smaller $C_L$, the filter function tends to be smoother especially in the high-spectrum domain. The filters fail to discriminate different high spectral components with similar frequency responses given to them.

In the neural network architectures, we consider assumptions on both the nonlinear activation function and the loss function as presented in Assumption \ref{ass:activation} and \ref{ass:loss} respectively. We note that these are reasonable assumptions as most activations (e.g. ReLU, modulus, sigmoid) and loss functions (e.g. L1 regression, Huber loss, quantile loss).

\begin{assumption}(Normalized Lipschitz activation functions)\label{ass:activation}
 The activation function $\sigma$ is normalized Lipschitz continuous, i.e., $|\sigma(a)-\sigma(b)|\leq |a-b|$, with $\sigma(0)=0$.
\end{assumption}

\begin{assumption}(Normalized Lipschitz loss function)\label{ass:loss}
 The loss function $\ell$ is normalized Lipschitz continuous, i.e., $|\ell(y_i,y)-\ell(y_j,y)|\leq |y_i-y_j|$, with $\ell(y,y)=0$.
\end{assumption}

The generalization gap is evaluated between the \emph{empirical risk} over the discrete graph model and the \emph{statistical risk} over the manifold model, which has been studied in both node-level and graph level in previous works when no model mismatch is considered \cite{wangcervnino2024neurips, wangcervino2024}. 

Suppose the training graph $\bbG$ is generated from a manifold $\ccalM$. The empirical risk that we train the GNN to minimize is therefore defined as 
\begin{align}
    \label{eqn:empirical-graphloss-node}
    R_\bbG(\bbH) =  \frac{1}{N}\sum_{i=1}^N\ell \left(  [\bm\Phi(\bbH, \bbL, \bbx )]_i  ,  [\bby ]_i \right).
\end{align}

The statistical risk over mismatched manifold $\ccalM_\tau$ is 
\begin{align}
    \label{eqn:statistical-manifoldloss-node}
    R_{\ccalM^\tau}(\bbH) = \int_{\ccalM^\tau} \ell \left( \bm\Phi(\bbH, \ccalL_\tau, f)(x), g(x)\right) \text{d}\mu_\tau(x).
\end{align}

The generalization gap under model mismatch is 
\begin{align}
    \label{eqn:generalization-gap}
    GA_\tau  =  \sup_{\bbH \in \ccalH } \left| R_{\ccalM^\tau}(\bbH) - R_\bbG(\bbH)\right|.
\end{align}

\begin{theorem}
\label{thm:generalization-node}
    Suppose a neural network is equipped with filters defined in Definition \ref{def:low-pass} and normalized nonlinearites (Assumption \ref{ass:activation}). The neural network is operated on a graph $\bbG$ generated according to \eqref{eqn:edge-weight} from $\ccalM$ and a mismatched manifold $\ccalM^\tau$ with a bandlimited (Definition \ref{def:band}) input manifold function. Suppose the mismatch $\tau :\ccalM\rightarrow \ccalM$ satisfies $\text{dist}(x,\tau(x))\leq \gamma$ and $\|J_x(\tau)-I\|_F\leq \gamma$ for all $x\in\ccalM$. The generalization of neural network trained on $\bbG$ to minimize a normalized Lipschitz loss function (Assumption \ref{ass:loss}) holds in probability at least $1-\delta$ that 
    \begin{equation}
    \label{eqn:ga-node}
        GA_\tau \leq  C_1\frac{\epsilon }{\sqrt{N}}
     + C_2 \frac{\sqrt{\log(1/\delta)}}{N}   +    C_3 \left(\frac{\log N}{N}\right)^{\frac{1}{d}} + C_4 \gamma,
    \end{equation}
    with $\epsilon \sim \left(\frac{\log(C/\delta)}{N} \right)^{\frac{1}{d+4}}$, $C_1$ scaling with $C_L$, $C_4$ scaling with $C_L$ and  $C_g$. $C_1$ and $C_3$ depend on the geometry of $\ccalM$.   

\end{theorem}
\begin{proof}
    See supplementary material for proof and the definitions of $C_1$, $C_2$, $C_3$ and $C_4$.
\end{proof}
\subsubsection{Remark 1} This conclusion is ready to extend to multi-layer and multi-feature neural network architectures, as the neural network is cascaded by layers of filters and nonlinearities. The generalization error propagates across layers, leading to an increase in the generalization gap of multi-layer and multi-feature GNNs with the size of the architecture, which we will further verify in simulations.

Theorem \ref{thm:generalization-node} indicates that the generalization of GNNs is robust to model mismatches, with the generalization gap decreasing with the number of nodes $N$ in the training graph. As more points are sampled from the manifold, the generated graph can better approximate the underlying manifold. This leads to a better generalization of the GNN to predict the unseen points over the manifold. The upper bound also increases with the dimension of the underlying manifold $d$, as higher dimension indicates higher model complexity. Specifically, the generalization gap increases with the mismatch size $\gamma$ as the testing manifold is shifted more from the original manifold. 
\subsubsection{Remark 2}It is important to note that there is a trade-off involved in designing GNNs: achieving better generalization often comes at the cost of reduced discriminability. Specifically, using a smaller $C_L$ leads to improved generalization and robustness by reducing generalization error. While this leads to smoother filter functions, which limits the GNN ability to discriminate between different spectral components. This reduced discriminability can negatively impact prediction performance. In essence, we can enhance better and more robust generalization capabilities of GNNs, but this improvement necessitates a compromise in their discriminative power.

\subsection{Extension to Graph Classification}
The generalization ability of GNN can be extended from the node level to the graph level, which indicates the manifold classification with approximated graph classification. 

Suppose we have manifolds $\{\ccalM^{\tau_k}_k\}_{k=1}^K$ with dimension $d_k$ under a mismatch $\tau_k$. Manifold $\ccalM^{\tau_k}_k$ is labeled with $y_k \in \reals$. The manifolds are smooth, compact, differentiable, and embedded in $\reals^\mathsf{M}$ with measure $\mu_{\tau_k,k}$. Each manifold $\ccalM^{\tau_k}_k$ is equipped with a weighted Laplace operator $\ccalL_{\tau_k,k}$. Assume that we can access $N_k$ randomly sampled points according to $\mu_k$ over each  manifold $\ccalM_k$. Graphs generated based on these sampled points are denoted as $\{\bbG_k\}_{k=1}^K$ with graph Laplacians $\{\bbL_k\}_{k=1}^K$. A GNN $\bm\Phi(\bbH,\bbL_{\cdot},\bbx_\cdot)$ is trained on these graphs with $\bbx_k$ denoted as the input data on graphs sampled from the data on manifolds $f_k\in L^2(\ccalM_k)$. The output of the GNN is set as the average of the output values over all the nodes while the output of MNN $\bm\Phi(\bbH,\ccalL_{\tau_\cdot,\cdot},f_\cdot)$ is the averaged value over the mismatched manifold. Loss function $\ell$ evaluates the performance of GNN and MNN by comparing the difference between the output label and the target label. The empirical risk that the GNN is trained to minimize is defined as 
\begin{align}
    \label{eqn:empirical-graphloss-graph}
    R_\bbG(\bbH) =  \sum_{k=1}^K \ell\left( \frac{1}{N_k}\sum_{i=1}^{N_k}  [\bm\Phi(\bbH, \bbL_{k}, \bbx_k)]_i  ,  y_k \right).
\end{align}

The statistical risk is defined based on the statistical MNN output and the target label over mismatched models as 
\begin{align}
    \label{eqn:statistical-manifoldloss-graph}
    R_{\ccalM^\tau}(\bbH) = \sum_{k=1}^K\ell \left( \int_{\ccalM_k} \bm\Phi(\bbH, \ccalL_{\tau_k,k}, f_k)(x) \text{d}\mu_{\tau_k}(x), y_k\right).
\end{align}

The generalization gap is therefore
\begin{align}
    \label{eqn:generalization-gap-mani}
    GA_\tau = \sup_{\bbH\in \ccalH} \left| R_{\ccalM^\tau}(\bbH) -  R_\bbG(\bbH)\right|.
\end{align}
\begin{theorem}
\label{thm:generalization-graph}
    Suppose the GNN and MNN with filters defined in Definition \ref{def:low-pass} and the input manifold functions are bandlimited (Definition \ref{def:band}). Suppose the mismatches $\tau_k :\ccalM_k\rightarrow \ccalM_k$ where $\text{dist}(x,\tau_k(x))\leq \gamma$ and $\|J_x(\tau_k)-I\|_F\leq \gamma$ for all $x\in\ccalM_k$ for $k=1,2,\cdots K$. Under Assumptions \ref{ass:activation} and \ref{ass:loss} it holds in probability at least $1-\delta$ that
    \begin{align}
   & \nonumber GA_\tau \leq  \sum_{k=1}^K \left(\frac{C_1}{\sqrt{N_k}} \epsilon_k 
     + C_2 \frac{\sqrt{\log(1/\delta)}}{N_k} \right) +    C_3 \sum
_{k=1}^K \left(\frac{\log N_k}{N_k}\right)^{\frac{1}{d_k}} + KC_4 \gamma,
    \end{align}
     with $\epsilon_k \sim \left(\frac{\log(C/\delta)}{N_k} \right)^{\frac{1}{d_k+4}}$, $C_1$ and $C_4$ scaling with $C_L$. $C_1$ and $C_3$ depend on the geometry of $\ccalM$. 
\end{theorem}
\begin{proof}
    See supplementary material for proof and the definitions of $C_1$, $C_2$, $C_3$ and $C_4$.
\end{proof}

Theorem \ref{thm:generalization-graph} indicates that a graph with large enough points sampled from each underlying manifold can approximately predict the label for mismathed manifolds. This shows that GNN trained over these generated graphs can generalize to classify unseen graphs generated from mismatched manifold models. The generalization gap on the graph level also decreases with the number of sampled points over each manifold. The generalization gap increases with the dimensions of the manifolds and the size of deformations. A similar trade-off phenomenon can also be observed.

\input{simulations}

\section{Conclusion}
We showed that the robustness of GNN generalization to model mismatch from a manifold perspective. We focused on graphs derived from manifolds, and we established that GNNs equipped with low-pass and integral Lipschitz graph filters exhibit robust generalization. This generalization extends to previously unseen nodes and graphs derived from mismatched manifolds. Notably, we identified a tradeoff between the robust generalization and discriminability of GNNs. We validate our results both on node-level and graph-level classification tasks with real-world datasets.

\urlstyle{same}
\bibliographystyle{plain}
\bibliography{Robust}

\clearpage 
\appendix
\input{appendix}

\end{document}

%% file: simulations.tex
\begin{figure*}[t]
    \centering
    \begin{subfigure}[b]{0.45\textwidth}
        \centering
        \includegraphics[width=\linewidth]{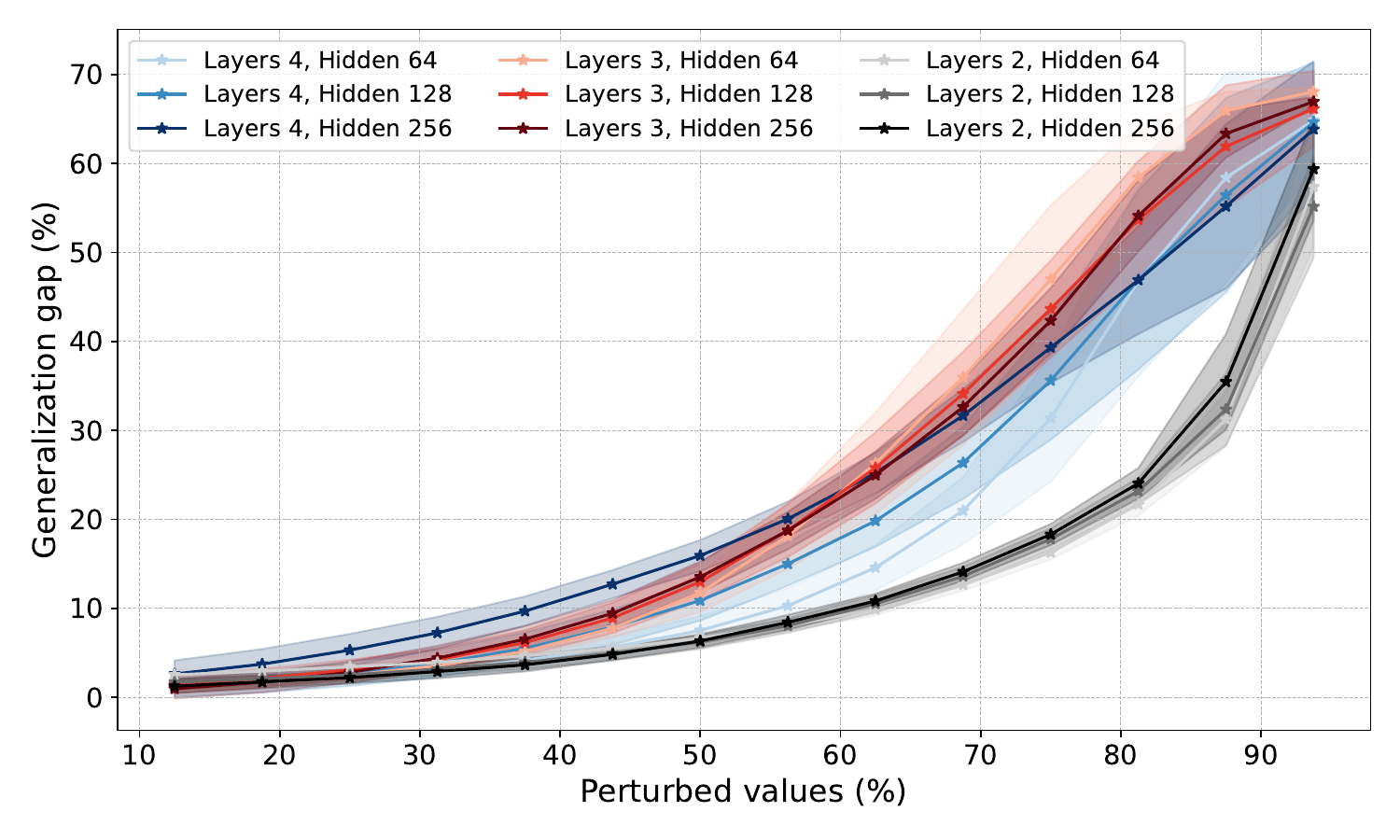}
        \caption{$90941$ nodes}
        \label{fig:edge_perturbation_90941}
    \end{subfigure}
    \hfill
    \begin{subfigure}[b]{0.45\textwidth}
        \centering
        \includegraphics[width=\linewidth]{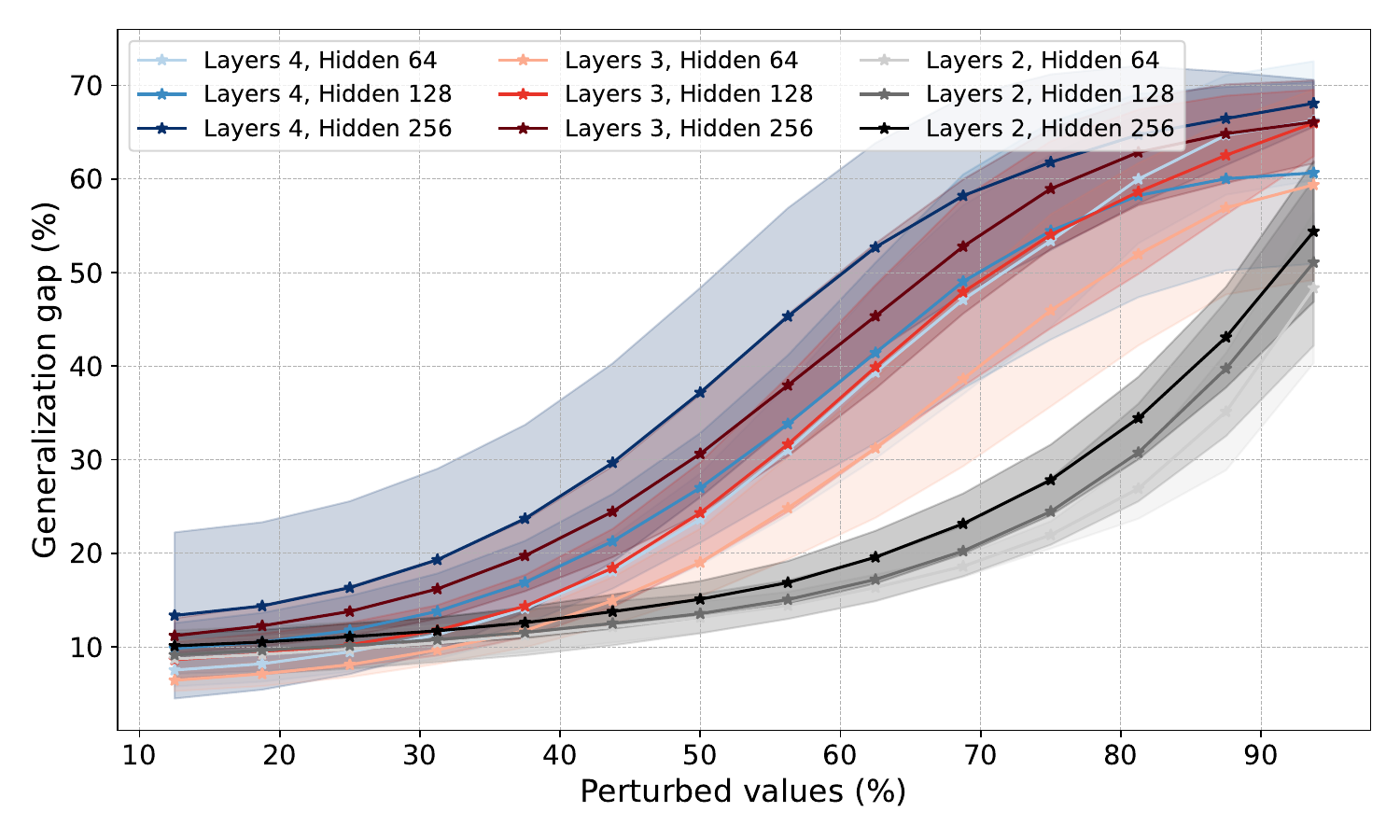}
        \caption{$2841$ nodes}
        \label{fig:node_perturbation_2841}
    \end{subfigure}
    \caption{Generalization gap as a function of the percentage of perturbed feature values in node feature perturbation. 
    }
    \label{fig:node_perturbation_two_figures}
\end{figure*}
\begin{figure*}[h]
    \centering
    \begin{subfigure}[b]{0.45\textwidth}
        \centering
        \includegraphics[width=\linewidth]{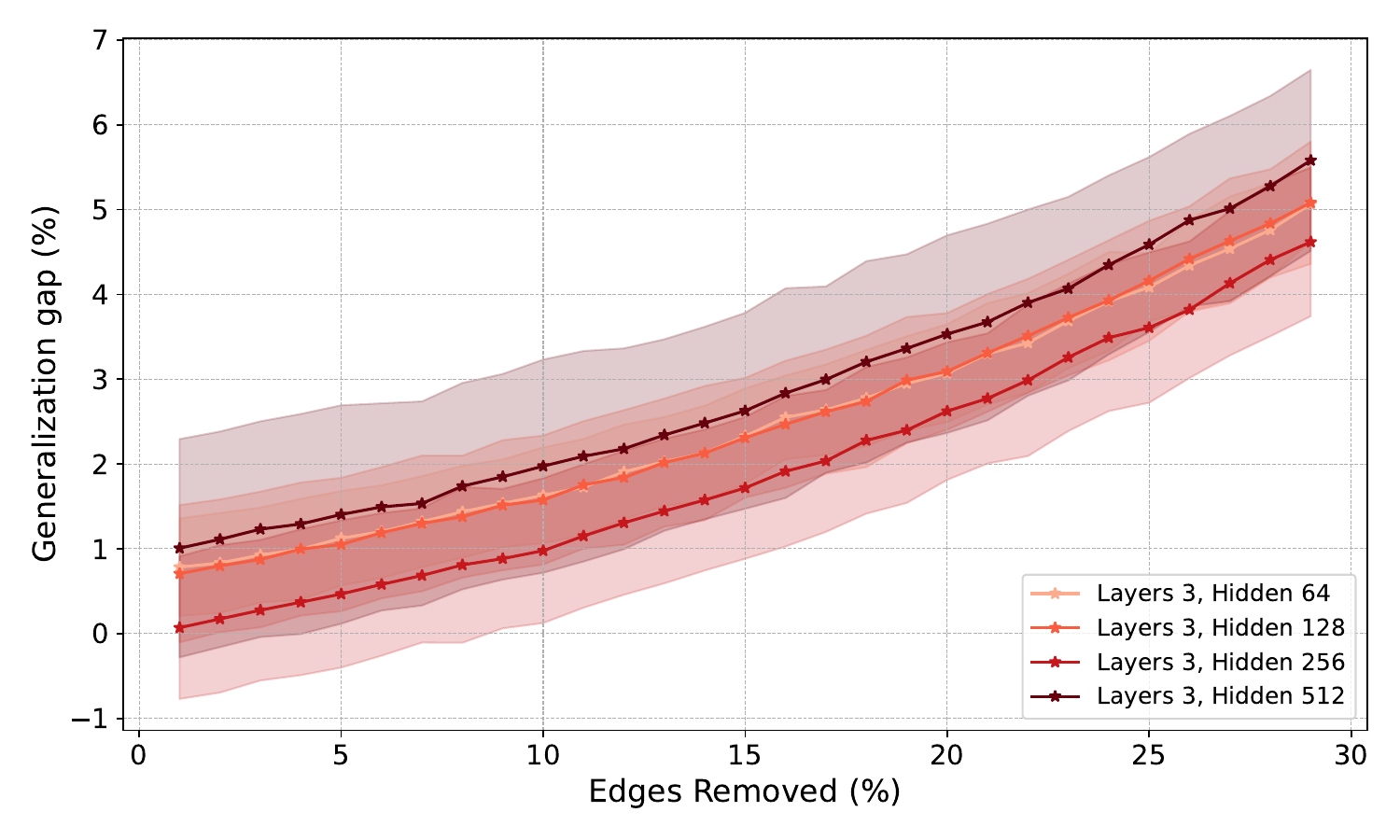}
        \caption{$3$ Layered GNN}
        \label{fig:edge_3_layers}
    \end{subfigure}
    \hfill
    \begin{subfigure}[b]{0.45\textwidth}
        \centering
        \includegraphics[width=\linewidth]{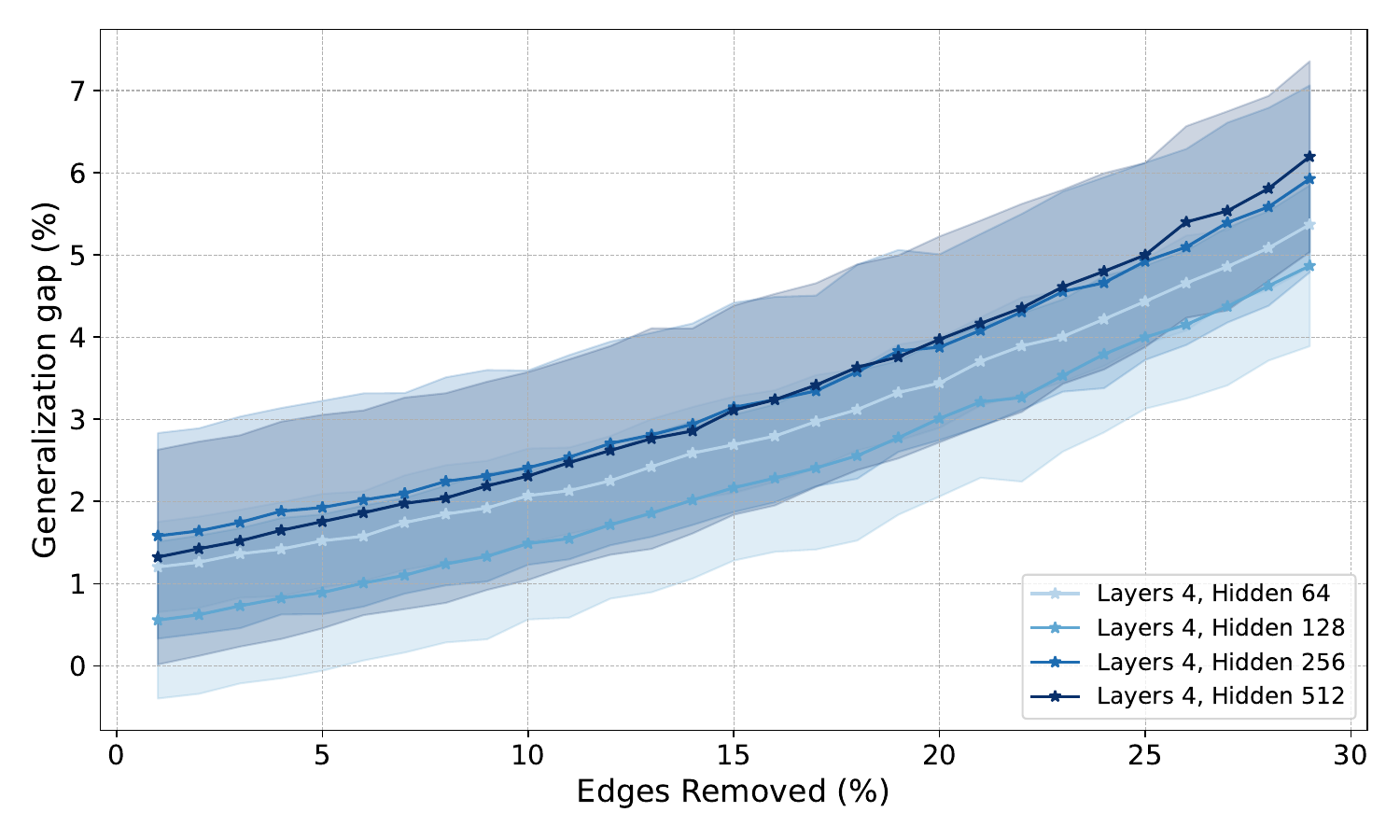}
        \caption{$4$ Layered GNN}
        \label{fig:edge_4_layers}
    \end{subfigure}
    \caption{Generalization gap as a function of the percentage of perturbed edges values in node removal perturbation. 
    }
    \label{fig:edge_perturbation_two_figures}
\end{figure*}

\section{Experiments}
\subsection{Node classification with Arxiv Dataset}
In this section, we showcase the generalization properties of a trained GNN on a real-world dataset,  OGBN-Arxiv \cite{wang2020microsoft}. The graph has $169,343$ nodes and $1,166,243$ edges, representing the citation network between computer science arXiv papers. The node features are $128$ dimensional embeddings of the title and abstract of each paper \cite{NIPS2013_9aa42b31}. The objective is to predict which of the $40$ categories the paper belongs to. We consider two types of perturbations to model the impact of the underlying manifold model mismatch -- node and edge perturbations. 

In all cases, we train the GNN on the original graph, and we evaluate it on the perturbed graph. The generalization gap is measured by the difference between the training accuracy and the perturbed testing accuracy. For node perturbations, we randomly zero out a percentage of the $128$ dimensional embeddings for all the points in the dataset. 
Theoretically, that corresponds to modifying the underlying scalar function $f$ (see equation \eqref{eqn:perturb-node}). 
For the edge perturbation, we randomly remove a percentage of the existing edges in the graph. 
This corresponds to perturbing the Laplace operator $\ccalL$ (see equation \eqref{eqn:perturb-edge}). 
%These two perturbations showcase the impact of a disturbance on the underlying convolutional operator of the GNN. 
%To showcase the statistical generalization analysis of GNNs under mismatch, we compared the training accuracy of the GNN on the unperturbed training set, with the test accuracy of the GNN on the perturbed graph. 
We run these experiments for a varying number of nodes, by partitioning the training set in $\{1, 2, 4, 8, ... 512, 1024\}$ partitions. We train the GNN with the \textit{cross-entropy} loss, for $1000$ epochs, using $0.005$ learning rate, ADAM optimizer \cite{kingma2014adam}, and no weight decay with \texttt{ReLU} non-linearity. The purpose of the experiments is to show that GNN with more layers and more hidden units has a larger generalization gap under the same perturbation level. We further verify the relationship between the generalization gaps and the logarithm of the number of nodes in the graphs as indicated in Theorem \ref{thm:generalization-node}. 
  
% To showcase the statistical generalization of GNNs to node feature perturbation, we randomly zero a percentage of the input feature information ($128$ dimensional) of every element on the graph. For every node, we zero different features, which can be interpreted as local perturbations of the graph. Thus showcasing the ability of GNNs to generalize under local mismatch. 

\begin{figure*}[t]
    \centering
    \begin{subfigure}[b]{0.45\textwidth}
        \centering
        \includegraphics[width=\linewidth]{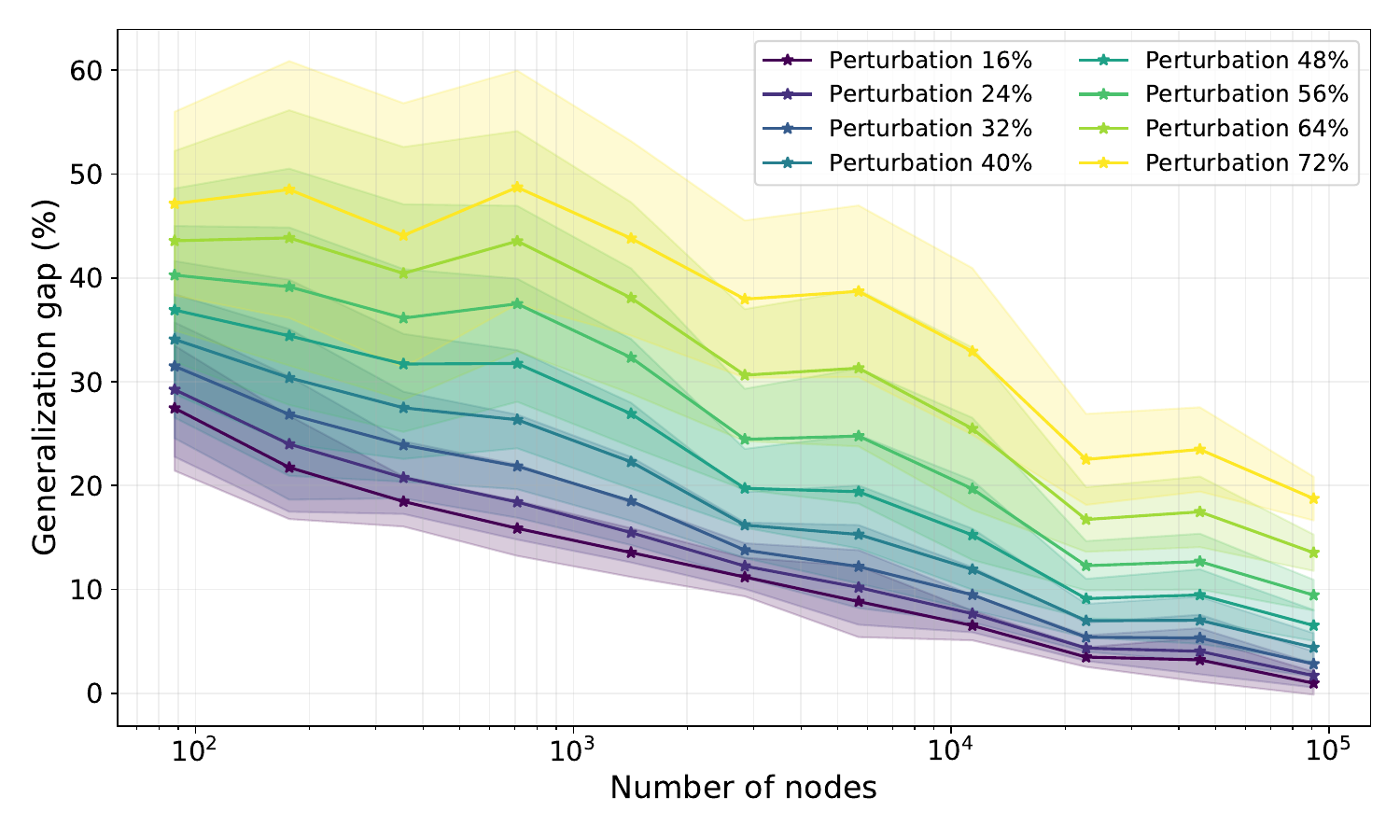}
        \caption{Node Feature Perturbation}
        \label{fig:node_perturbation}
    \end{subfigure}
    \hfill
    \begin{subfigure}[b]{0.45\textwidth}
        \centering
        \includegraphics[width=\linewidth]{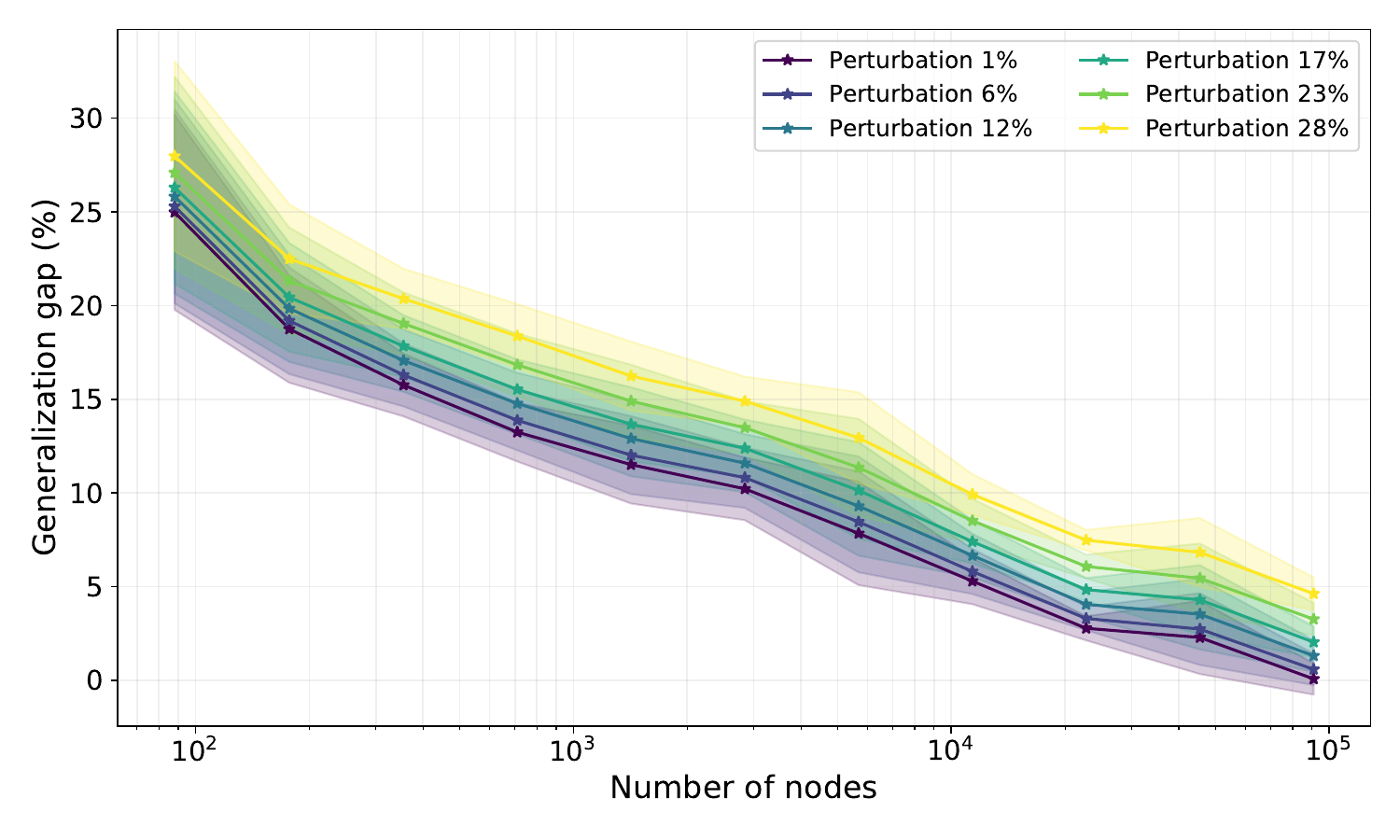}
        \caption{Edge Removal Perturbation}
        \label{fig:edge_perturbation}
    \end{subfigure}
    \caption{Generalization gap for edge and node perturbation for the Arxiv dataset for a $3$ layered, $256$ feature GNN. 
    }
    \label{fig:edge_node_perturbation}
\end{figure*}

\begin{figure*}[h]
    \centering
    \begin{subfigure}[b]{0.5\textwidth}
        \centering
        \includegraphics[width=\linewidth]{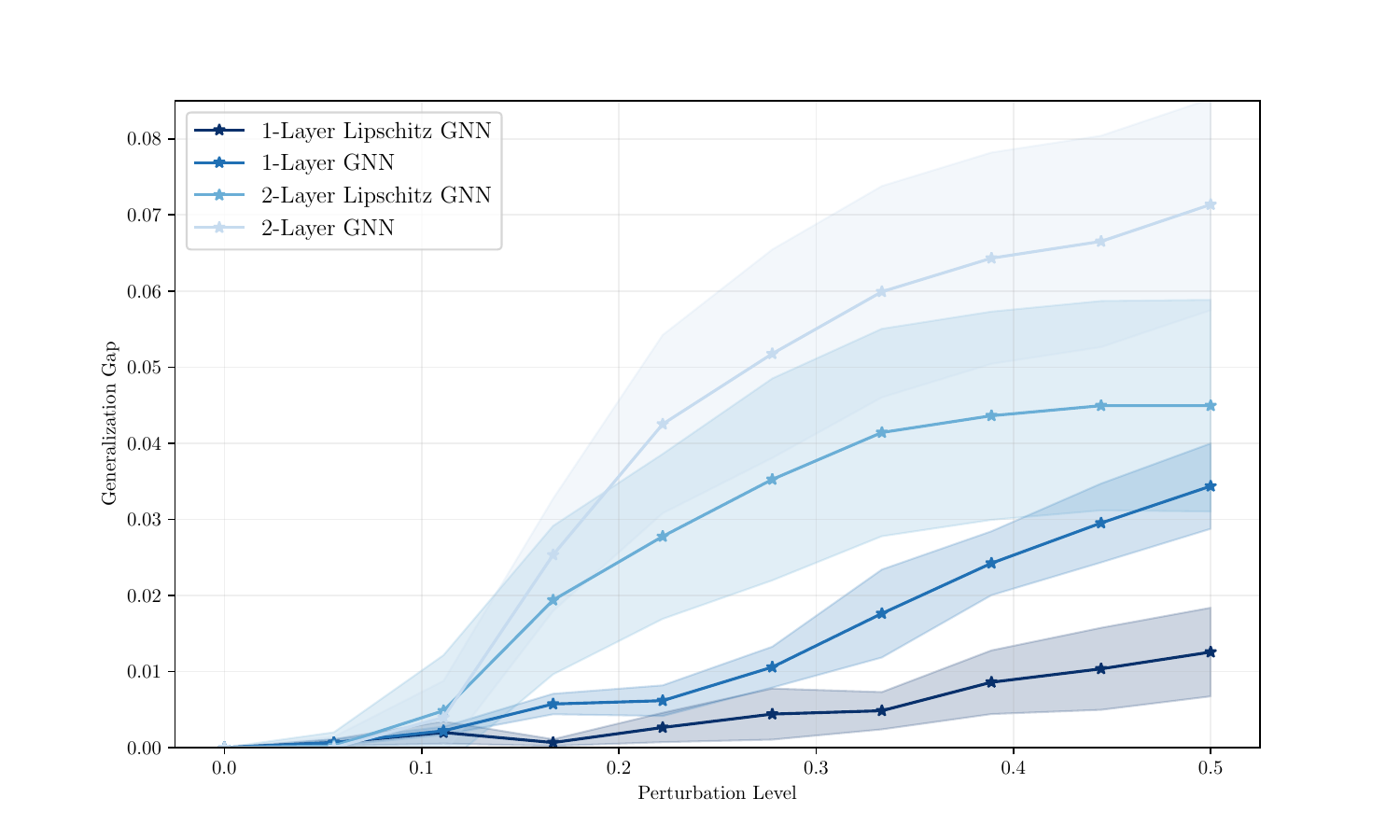}
        \caption{GNN with different architectures.}
        \label{fig:graph-modelnet}
    \end{subfigure}
    \hfill
    \begin{subfigure}[b]{0.49\textwidth}
        \centering
        \includegraphics[width=\linewidth]{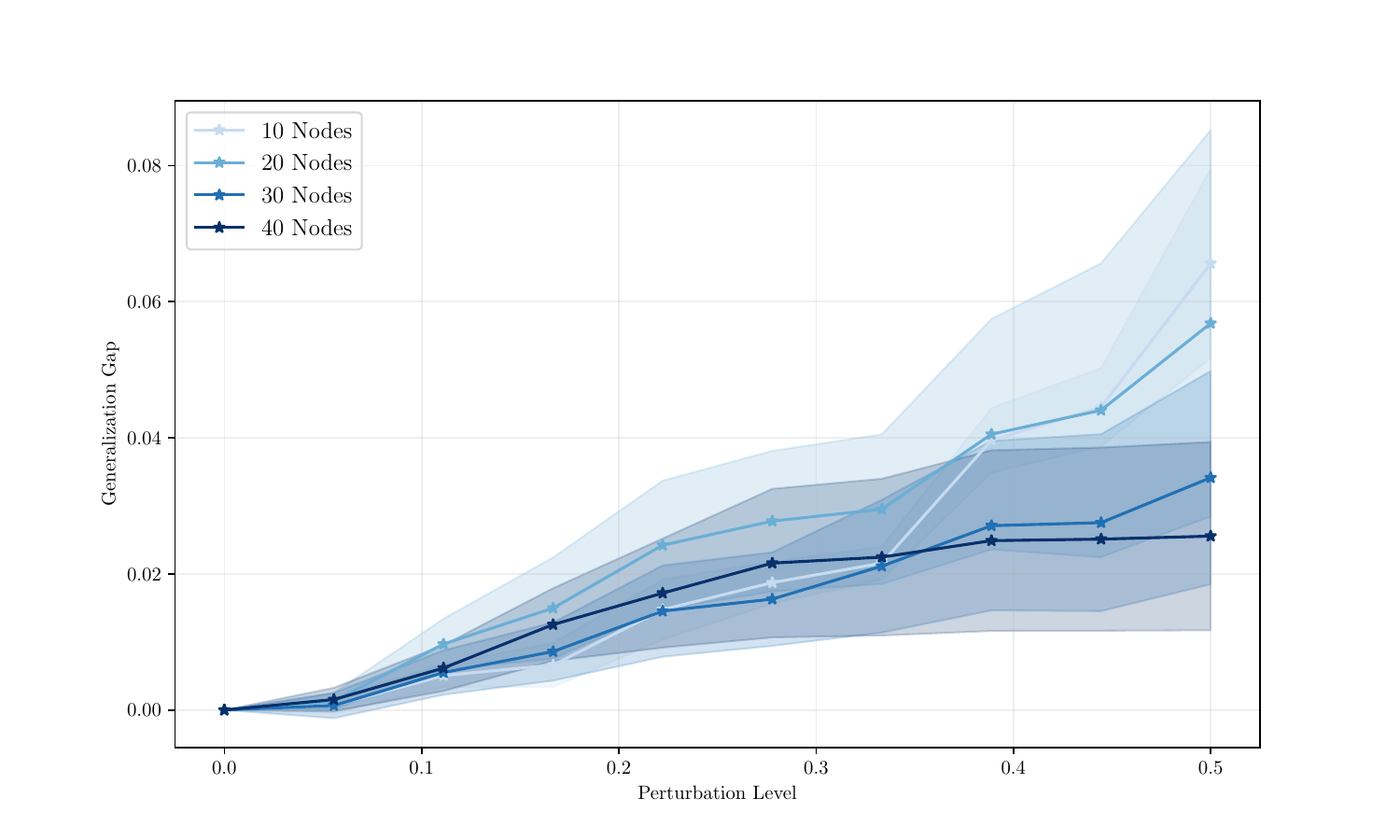}
        \caption{GNN on different number of nodes.}
        \label{fig:graph-node-modelnet}
    \end{subfigure}
    \caption{Generalization gap as a function of number of nodes in the Point Cloud Classification.
    }
    \label{fig:graph_class_perturbation}
\end{figure*}
In Figure \ref{fig:node_perturbation_two_figures} we can see the generalization gap as a function of the perturbation level for different GNN architectures. In these two figures, we confirm that GNNs with fewer layers and fewer hidden units are more robust to changes in the underlying manifold. This can be seen as the black lines ($2$ layers), are below both the red ($3$ layers) and the blue lines($4$ layers). Intuitively, this can be attributed to overfitting the training set, given that more capacity in the GNN translates into a better training set accuracy. We also showcase that if the number of nodes in the training set is larger, the GNN is more robust to changes on the graph, given that the generalization gap increases more slowly in the GNN trained with $90941$ nodes (Figure \ref{fig:edge_perturbation_90941}), than in the one trained with $2841$ nodes (Figure \ref{fig:node_perturbation_2841}).  In Figure \ref{fig:edge_perturbation_two_figures}, we plot the generalization gap as a function of the perturbation magnitude for a GNN with $3$ (Figure \ref{fig:edge_3_layers}), and $4$ layers (Figure \ref{fig:edge_4_layers}). In both cases, the GNN with the largest number of hidden units is at the top, thus indicating that a  larger generalization gap as indicated in Remark 1.

Figure \ref{fig:edge_node_perturbation} shows the generalization gaps of GNNs under node feature perturbations and edge perturbations decrease approximately linearly with the number of nodes in the logarithmic scale while increasing with the perturbation sizes. In Figure \ref{fig:node_perturbation} and \ref{fig:edge_perturbation}, we show the generalization gap of a $3$ layered GNN with $256$ hidden units as a function of the number of nodes in the training set under node feature perturbation and edge removal perturbation, respectively. Each color in the figure represents a different perturbation level, going from less perturbation ($16\%$ darker blue) to higher perturbation ($72\%$ lighter yellow). 
As can be seen in the plot, the GNNs generalization degrades as the perturbation level increases. Aligning with our theory, the generalization gap is linear with respect to the logarithm of the number of nodes. This is also true when the perturbation level increases. 
%In Figure , we show the linear relationship of the generalization gap as a function of the number of nodes in the graph. In this case, we considered a GNN with $3$ layers, and $256$ hidden units and varied the perturbation level, going from fewer removed nodes (darker blue $16\%$) to higher number of removed nodes (light yellow $28\%$). Consistent with our theory (see Theorem \ref{thm:generalization-node}), the generalization gap is linear with respect to the logarithm on the number of nodes in the graph. This is true even as the perturbation level increases.

% \subsection{Graph Perturbation - Node Removal}

% \subsubsection{Edge Perturbation}
% To showcase the impact of graph perturbations on the statistical generalization of GNNs, we randomly removed edges from the graph and evaluated the trained GNN on it. We removed a percentage of the total edges from the GNN, and evaluated the generalization gap for different numbers of layers, hidden units, and training nodes in the graph.

\subsection{Graph Classification with ModelNet Dataset}

% \begin{figure}[h]    
%         \centering
%         \includegraphics[width=1.1\linewidth]{AnonymousSubmission/LaTeX/figures/lipschitz-graph.pdf}
%         \caption{Point Cloud Classification}
%         \label{fig:edge_3_layers} 
% \end{figure}
We evaluate the generalization gap of GNNs on graph level with the ModelNet10 dataset \cite{wu20153d}. The dataset contains 3991 meshed CAD models from 10 categories for training and 908 models for testing. For each model, discrete points are uniformly randomly sampled from all points of the model to form the graphs. Each point is characterized by the 3D coordinates as features. Each node in the graph can be modeled as the sampling point and each edge weight is constructed based on the distance between each pair of nodes.
The input graph features are set as the coordinates of each point, and the weights of the edges are calculated based on the Euclidean distance between the points. The mismatched point cloud model adds a Gaussian random variable with mean $\gamma$ and variance $2\gamma$ to each coordinate of every sampled point. This can be seen as the underlying manifold mismatch. We calculate the generalization gap by training GNNs on graphs with $N = 40$ sampled points, and plotting the differences between the testing accuracy on the trained graphs sampled from normal point clouds and the testing accuracy on the graphs sampled from deformed point clouds with perturbation level $\gamma$. 
We implement Graph Neural Networks (GNN) with 1 and 2 layers with a single layer containing $F_0=3$ input features which are the 3d coordinates of each point, $F_1=64$ output features and $K=5$ filter taps. While the architectures with 2 layers has another layer with $F_2= 32$ features and $5$ filter taps. We use the \texttt{ReLU} as nonlinearity. All architectures also include a linear readout layer mapping the final output features to a binary scalar that estimates the classification. 
Figure \ref{fig:graph-modelnet}  shows the generalization gaps for GNNs with or without Lipschitz continuity assumptions of the graph filters. We observe that the continuity assumptions imposed on the graph filters help to improve the generalization capabilities as Theorem \ref{thm:generalization-graph} shows.
Figure \ref{fig:graph-node-modelnet} shows the results of GNNs on different number of nodes. We can see that the generalization gaps of GNNs scale with the size of the GNN architectures and scale with the size of mismatch. Figure \ref{fig:graph-node-modelnet} also shows that the generalization gap decreases with the number of nodes.  
% \subsection{Graph Perturbation - Edge Insertion }

%% file: appendix.tex
\section{Experiment Details}
All experiments were done with a \textit{NVIDIA GeForce RTX 3090}, and each set of experiments took at most $10$ hours to complete. The ArXiv dataset and ModelNet10 dataset used in this paper are public and free to use. They can be downloaded using the $pytorch$ package (\url{https://pytorch-geometric.readthedocs.io/en/latest/modules/datasets.html}), the \textit{ogb} package (\url{https://ogb.stanford.edu/docs/nodeprop/}) and the Princeton ModelNet project (\url{https://modelnet.cs.princeton.edu/}). In total, the datasets occupy around $5$ GB. However, they do not need to be all stored at the same time, as the experiments that we run can be done in series. We will release the codes after the paper is accepted.
\subsection{Node Classification}
we used the graph convolutional layer \texttt{GCN}, and trained for $1000$ epochs. For the optimizer, we used $\texttt{AdamW}$, with using a learning rate of $0.01$, and $0$ weight decay. We trained using the graph convolutional layer, with a varying number of layers and hidden units. For dropout, we used $0.5$. We trained using the cross-entropy loss. In all cases, we trained $2$ and $3$ layered GNNs. 
\subsection{ModelNet10 graph classification tasks}
ModelNet10 dataset \cite{wu20153d} includes 3,991 meshed CAD models from 10 categories for training and 908 models for testing as Figure \ref{fig:points-10} shows. In each model, $N$ points are uniformly randomly selected to construct graphs to approximate the underlying model, such as chairs, tables.

The weight function of the constructed graph is determined as (6) with $\epsilon = 0.001$. 
%Similarly, the weight function of an $\epsilon$-graph is calculated as \eqref{eqn:compact_kernel} with $\epsilon = 0.001$ as the threshold. 
We calculate the Laplacian matrix for each graph as the input graph shift operator. In this experiment, we implement GNNs with different numbers of layers and hidden units with $K=5$ filters in each layer. All the GNN architectures are trained by minimizing the cross-entropy loss. We implement an ADAM optimizer with the learning rate set as 0.005 along with the forgetting factors 0.9 and 0.999. We carry out the training for 40 epochs with the size of batches set as 10. We run 5 random dataset partitions and show the average performances and the standard deviation across these partitions. 

\begin{figure*}[t!]
\centering
\includegraphics[trim= 80 50 80  0,clip,width=0.2  \textwidth]{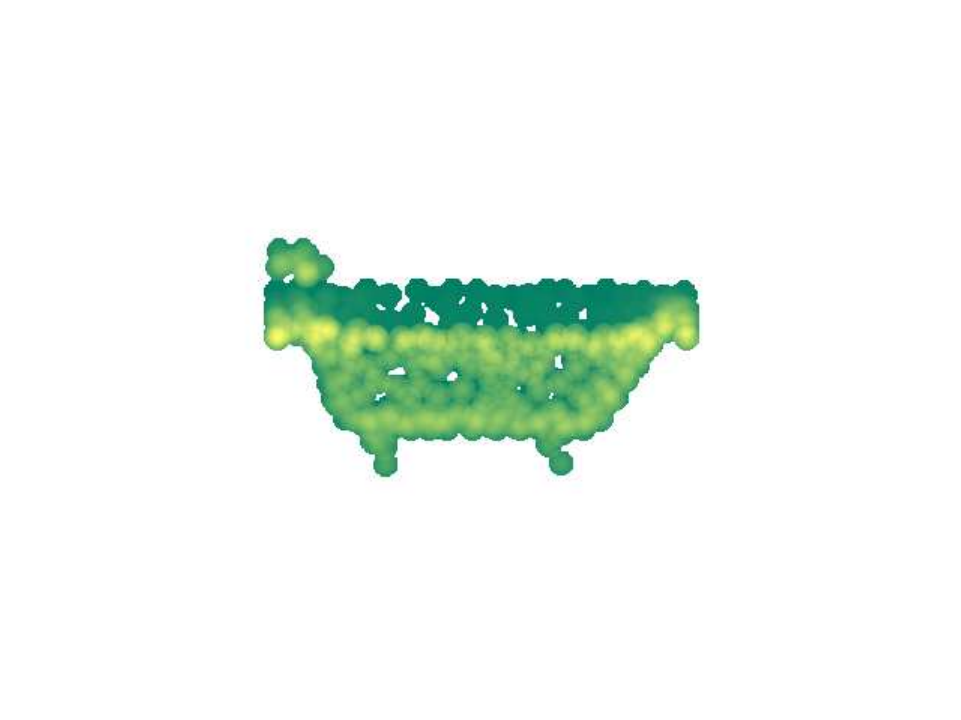}
\includegraphics[trim=90 0 100 0,width=0.15\textwidth]{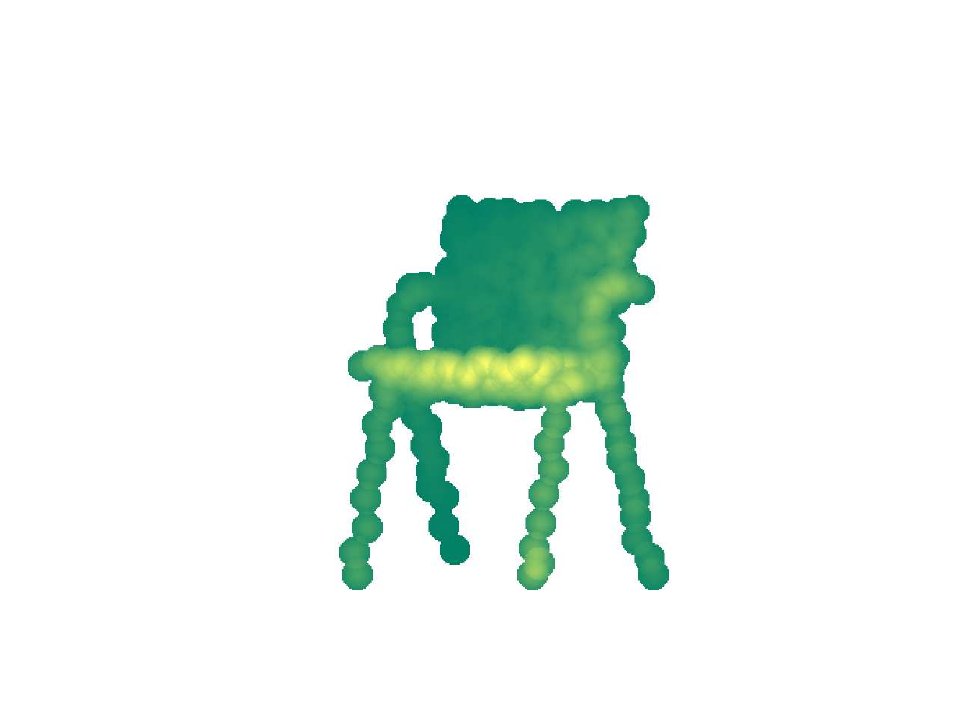}  
\includegraphics[trim=60 0 100 0,width=0.15\textwidth]{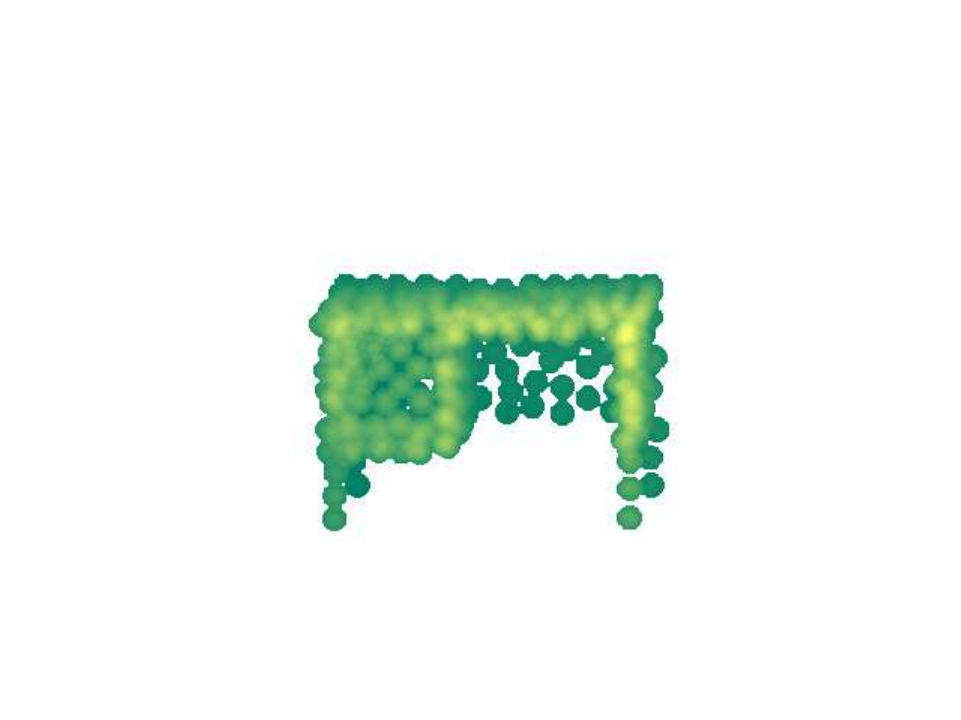} 
\includegraphics[trim=60 0 50 0,width=0.15\textwidth]{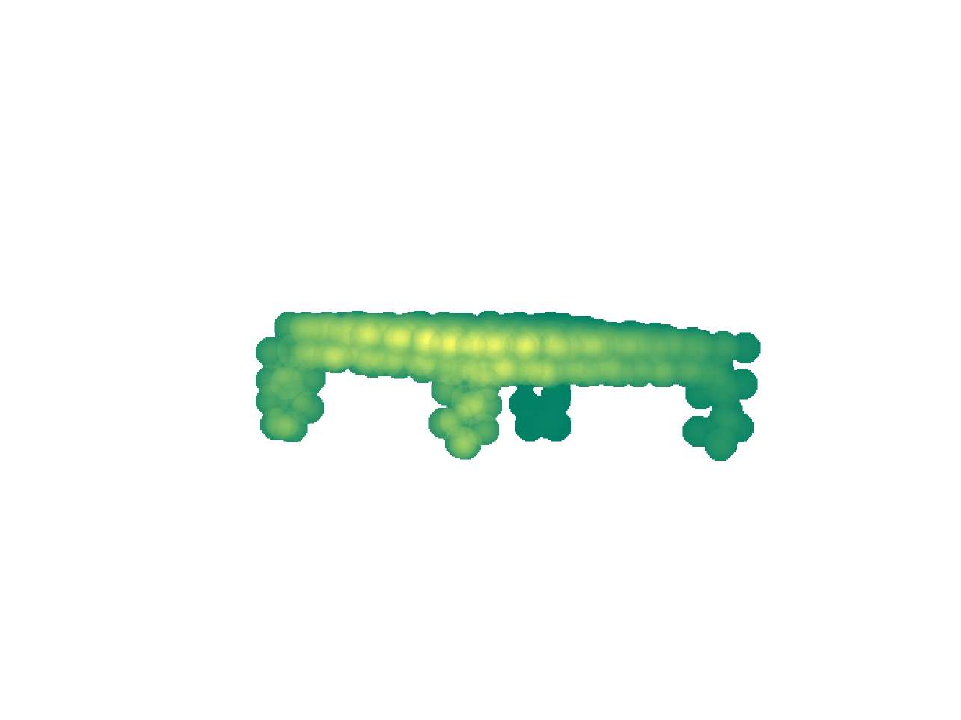}  
\includegraphics[trim=60 0 100 0,width=0.15\textwidth]{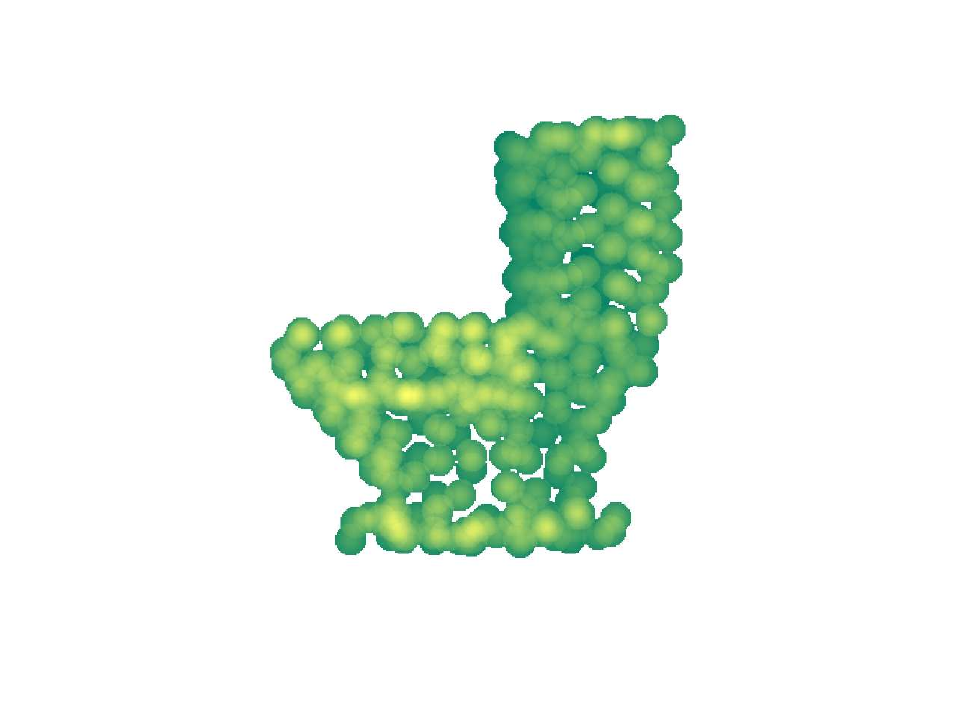}  
\includegraphics[trim=60 0 100 0,width=0.15\textwidth]{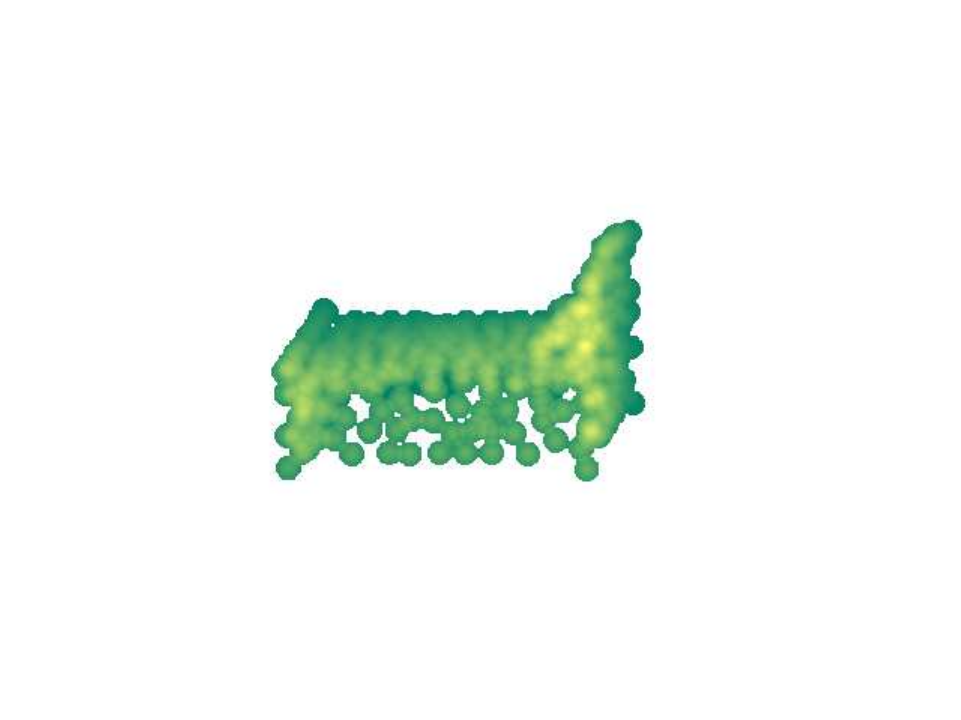} 

\caption{Point cloud models in ModelNet10 with $N=300$ sampled points in each model, corresponding to bathtub, chair, desk, table, toiler, and bed.}
\label{fig:points-10}
\end{figure*}

\section{Low Pass Filter Assumption}
In the main results, we assume that the GNN and MNN are low-pass filters. We believe this is a mild assumption. A high frequency on the graph/manifold implies that there exist signals (eigenvectors/functions) that can have large variations in adjacent entries. This naturally generates instabilities, which are more difficult to learn. It is expected certain amount of homogeneity within the neighborhood of a node, translates into low-frequency signals. It has been shown in practice that several real-world physical models lead to low-pass filters such as opinion dynamics \cite{LPF_consensus}, econometrics \cite{billio2012econometric}, and graph signal processing \cite{ramakrishna2020user}. 

Several other learning techniques rely on low-pass filters. Principal Component Analysis (PCA), can be seen as low-pass filters applied to the correlation matrix. Isomap, is another example of low pass filtering, but in this case applied to the Laplacian matrix. In all, we believe that the low pass filter assumption is not a hard assumption, given its empirical verification, and the fact that several other techniques that work in practice utilize it.

\section{Proofs}
\subsection{Proof of Theorem 1}
\label{app:proof-node}
The generalization error of GNNs under deformations is defined as 
\begin{align}
    \label{eqn:generalization-gap}
    GA_\tau = \sup_{\bbH\in \ccalH} \left| R_{\ccalM^\tau}(\bbH) -  R_\bbG(\bbH)\right|.
\end{align}
% Suppose $\bbH \in \arg\min_{\bbH\in\ccalH} R_S^\tau(\bbH)$, we have 
% \begin{align}
%     & GA_\tau = \min_{\bbH\in \ccalH} R_E(\bbH)- \min_{\bbH\in \ccalH} R^\tau_S(\bbH) \leq R_E(\bbH) - R^\tau_S(\bbH)
% \end{align}

An intermediate term $\int_{\ccalM} \ell\left( {\bm\Phi}(\bbH,\ccalL, f)(x), {g}(x) \right)\text{d}\mu(x)$ is first subtracted and added as follows. 
    \begin{align}
         \left| R_{\ccalM^\tau}(\bbH) -  R_\bbG(\bbH)\right|  & = \Bigg|\frac{1}{N} \sum_{i=1}^N \ell([\bm\Phi(\bbH,\bbL, \bbx)]_i , [\bby]_i)    - \int_{\ccalM^\tau} \ell\left( {\bm\Phi}(\bbH,\ccalL_\tau, f)(x), {g}(x) \right)\text{d}\mu_\tau(x) \Bigg|\\
       & \nonumber = \Bigg|\frac{1}{N} \sum_{i=1}^N \ell([\bm\Phi(\bbH,\bbL, \bbx)]_i , [\bby]_i)   - \int_{\ccalM} \ell\left( {\bm\Phi}(\bbH,\ccalL, f)(x), {g}(x) \right)\text{d}\mu(x) \\ 
       & + \int_{\ccalM} \ell\left( {\bm\Phi}(\bbH,\ccalL, f)(x), {g}(x) \right)\text{d}\mu(x)   - \int_{\ccalM^\tau} \ell\left( {\bm\Phi}(\bbH,\ccalL_\tau, f)(x), {g}(x) \right)\text{d}\mu_\tau(x) \Bigg|
    \end{align}

This can be further decomposed with absolute value inequality as follows.
\begin{align}
   \left| R_{\ccalM^\tau}(\bbH) -  R_\bbG(\bbH)\right|  & \leq \Bigg|  \frac{1}{N} \sum_{i=1}^N \ell([\bm\Phi(\bbH,\bbL, \bbx)]_i , [\bby]_i)   - \int_{\ccalM} \ell\left( {\bm\Phi}(\bbH,\ccalL, f)(x), {g}(x) \right)\text{d}\mu(x) \Bigg| \\ 
    &\nonumber+ \Bigg| \int_{\ccalM} \ell\left( {\bm\Phi}(\bbH,\ccalL, f)(x), {g}(x) \right)\text{d}\mu(x)    - \int_{\ccalM^\tau} \ell\left( {\bm\Phi}(\bbH,\ccalL_\tau, f)(x), {g}(x) \right)\text{d}\mu_\tau(x) \Bigg|
    \label{eqn:1}
    \end{align}
The first part in \eqref{eqn:1} can be found in Theorem 1 of \cite{wangcervnino2024neurips} with the proof in Appendix D. With the conclusion written as 
\begin{align}
    & \nonumber  \Bigg|  \frac{1}{N} \sum_{i=1}^N \ell([\bm\Phi(\bbH,\bbL, \bbx)]_i , [\bby]_i)   - \int_{\ccalM} \ell\left( {\bm\Phi}(\bbH,\ccalL, f)(x), {g}(x) \right)\text{d}\mu(x) \Bigg| \\
    & \qquad \qquad \qquad \leq C_1\frac{\epsilon }{\sqrt{N}}
     + C_2 \frac{\sqrt{\log(1/\delta)}}{N}   +    C_3 \left(\frac{\log N}{N}\right)^{\frac{1}{d}},
\end{align}
with 
$C_1 = C_{\ccalM,1}\frac{\pi^2}{6}\|f\|_\ccalM+ C_{\ccalM,2}\frac{\pi^2}{6}$. $C_{\ccalM,1}$, $C_{\ccalM,2}$ and $C_3$ depend on $d$ and the volume, the probability density, the injectivity radius and sectional curvature of $\ccalM$, $C_2 = \frac{\pi^2}{6}$.  

We focus on the integration on the mismatched manifold first.
\begin{align}
      \int_{\ccalM^\tau} \ell\left( {\bm\Phi}(\bbH,\ccalL_\tau, f)(x), {g}(x) \right)\text{d}\mu_\tau(x) 
   & = \int_{\ccalM}  \ell\left( {\bm\Phi}(\bbH,\ccalL, f)(\tau(x)), {g}(\tau(x)) \right)\text{d}\mu(\tau(x))\\
   & = \int_{\ccalM}  \ell\left( {\bm\Phi}(\bbH,\ccalL', f)(x), {g}(\tau(x)) \right) J_x(\tau) \text{d}\mu(x),
\end{align}
with $\ccalL' f(x) = \ccalL f(\tau(x))$. The second part of \eqref{eqn:1} therefore can be written as
\begin{align}
   & \nonumber \Bigg| \int_{\ccalM} \ell\left( {\bm\Phi}(\bbH,\ccalL, f)(x), {g}(x) \right)\text{d}\mu(x)    - \int_{\ccalM^\tau} \ell\left( {\bm\Phi}(\bbH,\ccalL_\tau, f)(x), {g}(x) \right)\text{d}\mu_\tau(x) \Bigg| \\
    &  =  \Bigg| \int_{\ccalM} \ell\left( {\bm\Phi}(\bbH,\ccalL, f)(x), {g}(x) \right)\text{d}\mu(x)    - \int_{\ccalM} \ell\left( {\bm\Phi}(\bbH,\ccalL', f)(x), {g}(\tau(x)) \right)J_x(\tau)\text{d}\mu(x) \Bigg| 
    \end{align}
By subtracting and adding an intermediate term $\int_{\ccalM} \ell\left( {\bm\Phi}(\bbH,\ccalL, f)(x), {g}(\tau(x)) \right)\text{d}\mu(x)$ and with the absolute value inequality, we have
\begin{align}
&\nonumber \Bigg| \int_{\ccalM} \ell\left( {\bm\Phi}(\bbH,\ccalL, f)(x), {g}(x) \right)\text{d}\mu(x)     - \int_{\ccalM^\tau} \ell\left( {\bm\Phi}(\bbH,\ccalL_\tau, f)(x), {g}(x) \right)\text{d}\mu_\tau(x) \Bigg| \\
    &\nonumber \leq \Bigg| \int_{\ccalM} \ell\left( {\bm\Phi}(\bbH,\ccalL, f)(x), {g}(x) \right)\text{d}\mu(x)   - \int_{\ccalM} \ell\left( {\bm\Phi}(\bbH,\ccalL, f)(x), {g}(\tau(x)) \right)\text{d}\mu(x) \Bigg| \\
    &   + \Bigg| \int_{\ccalM} \ell\left( {\bm\Phi}(\bbH,\ccalL, f)(x), {g}(\tau(x)) \right)\text{d}\mu(x)    - \int_{\ccalM} \ell\left( {\bm\Phi}(\bbH,\ccalL', f)(x), {g}(\tau(x)) \right)J_x(\tau)\text{d}\mu(x) \Bigg|
    \label{eqn:2}
\end{align}
The first part of \eqref{eqn:2} can be bounded by the Lipschitz continuity of loss function $\ell$, which leads to
\begin{align}
    &\nonumber \Bigg| \int_{\ccalM} \ell\left( {\bm\Phi}(\bbH,\ccalL, f)(x), {g}(x) \right)\text{d}\mu(x)    - \int_{\ccalM} \ell\left( {\bm\Phi}(\bbH,\ccalL, f)(x), {g}(\tau(x)) \right)\text{d}\mu(x) \Bigg| \\
    &   \leq  \int_\ccalM  \Big| \ell\left( {\bm\Phi}(\bbH,\ccalL, f)(x), {g}(x) \right)  - \ell\left( {\bm\Phi}(\bbH,\ccalL, f)(x), {g}(\tau(x)) \right)\Big|\text{d}\mu(x)  \leq  \int_\ccalM |g(x)-g(\tau(x))|\text{d}\mu(x).
\end{align}
With the Lipschitz continuity of target function $g$ and the bounded distance between $x$ and $\tau(x)$ as $dist(x,\tau(x))\leq \gamma$, we have
\begin{align}
    \int_\ccalM |g(x)-g(\tau(x))|\text{d}\mu(x) \leq C_g \gamma.
\end{align}
Therefore, the first part of \eqref{eqn:2} can be bounded as follows.
\begin{align}
    &  \Bigg| \int_{\ccalM} \ell\left( {\bm\Phi}(\bbH,\ccalL, f)(x), {g}(x) \right)\text{d}\mu(x)   - \int_{\ccalM} \ell\left( {\bm\Phi}(\bbH,\ccalL, f)(x), {g}(\tau(x)) \right)\text{d}\mu(x) \Bigg|\leq C_g \gamma.  
\end{align}
The second part of \eqref{eqn:2} can be first bounded by decomposing the Jacobian matrix as
\begin{align}
    &\nonumber \Bigg| \int_{\ccalM} \ell\left( {\bm\Phi}(\bbH,\ccalL, f)(x), {g}(\tau(x)) \right)\text{d}\mu(x)  - \int_{\ccalM} \ell\left( {\bm\Phi}(\bbH,\ccalL', f)(x), {g}(\tau(x)) \right)J_x(\tau)\text{d}\mu(x) \Bigg|\\
    & \nonumber = \Bigg| \int_{\ccalM} \ell\left( {\bm\Phi}(\bbH,\ccalL, f)(x), {g}(\tau(x)) \right)\text{d}\mu(x)   - \int_{\ccalM} \ell\left( {\bm\Phi}(\bbH,\ccalL', f)(x), {g}(\tau(x)) \right) (1 + \Delta_\gamma)\text{d}\mu(x) \Bigg|\\
    &\nonumber \leq \Bigg|  \int_{\ccalM} \ell\left( {\bm\Phi}(\bbH,\ccalL, f)(x), {g}(\tau(x)) \right) \text{d}\mu(x)    - \int_{\ccalM} \ell\left( {\bm\Phi}(\bbH,\ccalL', f)(x), {g}(\tau(x)) \right)\text{d}\mu(x) \Bigg|\\
     & \qquad \qquad \qquad  + \|\Delta_\gamma\|_F \int_{\ccalM}\ell\left( {\bm\Phi}(\bbH,\ccalL', f)(x), {g}(\tau(x)) \right) \text{d}\mu(x)\\
      &\nonumber \leq \Bigg|  \int_{\ccalM} \ell\left( {\bm\Phi}(\bbH,\ccalL, f)(x), {g}(\tau(x)) \right) \text{d}\mu(x)    - \int_{\ccalM} \ell\left( {\bm\Phi}(\bbH,\ccalL', f)(x), {g}(\tau(x)) \right)\text{d}\mu(x) \Bigg|\\
     & \qquad \qquad \qquad + \gamma \int_{\ccalM}\ell\left( {\bm\Phi}(\bbH,\ccalL', f)(x), {g}(\tau(x)) \right) \text{d}\mu(x),
    \end{align}
where the last step is bounded by the assumption that $\|J_x(\tau) - I\|_F\leq \gamma$.
With the Lipschitz continuity and the boundedness of loss function $\ell$, we have 
\begin{align}
     &  \nonumber \Bigg| \int_{\ccalM} \ell\left( {\bm\Phi}(\bbH,\ccalL, f)(x), {g}(\tau(x)) \right)\text{d}\mu(x)   - \int_{\ccalM} \ell\left( {\bm\Phi}(\bbH,\ccalL', f)(x), {g}(\tau(x)) \right)J_x(\tau)\text{d}\mu(x) \Bigg|\\
    &\leq \int_\ccalM|{\bm\Phi}(\bbH,\ccalL, f)(x) - {\bm\Phi}(\bbH,\ccalL', f)(x)|\text{d}\mu(x)+\gamma\ell_{max}\\
    &= \|{\bm\Phi}(\bbH,\ccalL, f)- {\bm\Phi}(\bbH,\ccalL', f)\|_1+\gamma\ell_{max}\\
    & \leq \sqrt{d} \|{\bm\Phi}(\bbH,\ccalL, f)- {\bm\Phi}(\bbH,\ccalL', f)\|_2+\gamma \ell_{max},\label{eqn:part1}
\end{align}
% \begin{align}
%      & \leq \|{\bm\Phi}(\bbH,\ccalL, f) - {\bm\Phi}(\bbH,\ccalL', f) \|_1 + \gamma 
% \end{align}

where the last inequality comes from the fact that $\|f\|_1 \leq \sqrt{d}\|f\|_2$ for $f\in L^2(\ccalM)$. The loss function value is bounded as the output of MNN and the target function are both bounded.
To study the difference of the MNN outputs, we import the stability property of manifold neural networks in \cite{wang2024stability}. Specifically, we utilize Theorem 1, which interprets the manifold deformation as perturbations of Laplacian operator.
\begin{theorem}[Theorem 1 in \cite{wang2024stability}] \label{thm:perturb}
{Let $\ccalL$ be the LB operator of manifold $\ccalM$.
Let $\tau:\ccalM\rightarrow \ccalM$ be a manifold perturbation such that $\text{dist}(x,\tau(x))\leq \gamma$ and $\|J_x(\tau)- I\|_F\leq \gamma$ for all $x\in \ccalM$. If the gradient field is smooth, it holds that}
\begin{equation}
    \label{eqn:perturb-operator}
    \ccalL-\ccalL' = \bbE \ccalL + \ccalA ,
\end{equation}
where $\bbE$ and $\ccalA$ satisfy $\|\bbE\|=O(\gamma)$ and $\|\ccalA\|_{op}= O(\gamma)$. 
\end{theorem}
We further import Lemma 1, Lemma 2 and Lemma 3 from \cite{wang2024stability}.
\begin{proposition}[Lemma 1 and Lemma 3 in \cite{wang2024stability}] \label{thm:perturb}
The eigenvalues of LB operators $\ccalL$ and perturbed $\ccalL'=\ccalL+\bbE\ccalL+\ccalA$, with $\|\bbE\|=O(\gamma)$ and $\|\ccalA\|_{op}= O(\gamma)$ satisfy
\begin{equation}
|\lambda_i-\lambda'_i|\leq \gamma|\lambda_i|+\gamma, \text{ for all }i=1,2\hdots
\end{equation}
\end{proposition}
Let $\beta = \min\{ \lambda_{M+1} -\lambda_{M},\gamma\lambda_2+\gamma \}$ with $M = \max\{ i| \lambda_i \leq \lambda\}$. We denote the set of eigenvalue indices with the difference compared to the previous eigenvalue larger than $\gamma$ as 
\begin{equation}
    I = \{1\}\cup\{i>1|\lambda_i - \lambda_{i-1} \geq \beta \},
\end{equation}
with $|I| = K$.
We define a new eigenvalue sequence with
\begin{equation}
    J_i =
    \begin{cases}
      i-1 & \text{if $I_k \leq i < I_{k+1}$, $k=1,2 \cdots ,K$}\\
      i & \text{otherwise}
    \end{cases}     
\end{equation}
We start with the difference of manifold filters by first writing out the spectral representation as 
\begin{align}
    \bbh(\ccalL)f = \sum_{i=1}^M \hat{h}(\lambda_i)\langle f,\bm\phi_i \rangle \bm\phi_i.
\end{align}
The difference can be written as 
\begin{align}
    | \bbh(\ccalL)f  - \bbh(\ccalL')f| = \left| \sum_{i=1}^M \hat{h}(\lambda_i)\langle f,\bm\phi_i \rangle \bm\phi_i -  \sum_{i=1}^M \hat{h}(\lambda_i')\langle f,\bm\phi_i' \rangle \bm\phi_i'\right|
\end{align}
We subtract and add an intermediate term $\sum_{i=1}^M \hat{h}(\lambda_{J_i})\langle f,\bm\phi_i \rangle \bm\phi_i$. With the absolute value inequality, we can decompose the difference as 
\begin{align}
     &\nonumber \| \bbh(\ccalL)f  - \bbh(\ccalL')f\| \\
    &  = \Bigg\| \sum_{i=1}^M \hat{h}(\lambda_i)\langle f,\bm\phi_i \rangle \bm\phi_i  - \sum_{i=1}^M \hat{h}(\lambda_{J_i})\langle f,\bm\phi_{i} \rangle \bm\phi_{i} + \sum_{i=1}^M \hat{h}(\lambda_{J_i})\langle f,\bm\phi_{i} \rangle \bm\phi_{i}
    -  \sum_{i=1}^M \hat{h}(\lambda_i')\langle f,\bm\phi_i' \rangle \bm\phi_i'\Bigg\|\\
    &  \label{eqn:3} \leq \left\| \sum_{i=1}^M \hat{h}(\lambda_i)\langle f,\bm\phi_i \rangle \bm\phi_i  - \sum_{i=1}^M \hat{h}(\lambda_{J_i})\langle f,\bm\phi_{i} \rangle \bm\phi_{i}\right\| + \left\|\sum_{i=1}^M \hat{h}(\lambda_{J_i})\langle f,\bm\phi_{i} \rangle \bm\phi_{i}
    -  \sum_{i=1}^M \hat{h}(\lambda_i')\langle f,\bm\phi_i' \rangle \bm\phi_i'\right\|.
\end{align}
The first term in \eqref{eqn:3} can be decomposed by subtracting and adding an intermediate term $\sum_{i=1}^{M}\hat{h}(\lambda_{J_i})\langle f,\bm\phi_{i} \rangle \bm\phi_{i}$. With absolute value inequality, we have
\begin{align}
   & \nonumber \left\| \sum_{i=1}^M \hat{h}(\lambda_i)\langle f,\bm\phi_i \rangle \bm\phi_i  - \sum_{i=1}^M \hat{h}(\lambda_{J_i})\langle f,\bm\phi_{ i} \rangle \bm\phi_{ i}\right\| \\
   &\label{eqn:part2-1}\leq \sum_{i=1}^M \left|\hat{h}(\lambda_i)-\hat{h}(\lambda_{J_i})\right|\|f\|_\ccalM\|\bm\phi_i\|^2\\
   &\leq \sum_{i=1}^M  \left|\hat{h}'(\lambda_i)\right||\lambda_i-\lambda_{J_i}| \|f\|_\ccalM\\
   & \leq \sum_{i=1}^M C_L \lambda_i^{-d-1} \beta \|f\|_\ccalM \leq \frac{\pi^2}{6}C_L \beta\|f\|_\ccalM. \label{eqn:part2}
 %   & \nonumber \leq \Bigg| \sum_{i=1}^M \hat{h}(\lambda_i)\langle f,\bm\phi_i \rangle \bm\phi_i  - \sum_{i=1}^M \hat{h}(\lambda_{J_i})\langle f,\bm\phi_{i} \rangle \bm\phi_{i}   \\
 % & + \sum_{i=1}^M \hat{h}(\lambda_{J_i})\langle f,\bm\phi_{i} \rangle \bm\phi_{i} - \sum_{i=1}^M \hat{h}(\lambda_{J_i})\langle f,\bm\phi_{J_i} \rangle \bm\phi_{J_i}
 %   \Bigg| \\
 %   & \nonumber\leq \left|\sum_{i=1}^M \hat{h}(\lambda_i)\langle f,\bm\phi_i \rangle \bm\phi_i  - \sum_{i=1}^M \hat{h}(\lambda_{J_i})\langle f,\bm\phi_{i} \rangle \bm\phi_{i}  \right| \\
 %   &\label{eqn:4} + \left| \sum_{i=1}^M \hat{h}(\lambda_{J_i})\langle f,\bm\phi_{i} \rangle \bm\phi_{i} - \sum_{i=1}^M \hat{h}(\lambda_{J_i})\langle f,\bm\phi_{J_i} \rangle \bm\phi_{J_i} \right|.
\end{align}
As $M$ becomes very large, Weyl's law which indicates that eigenvalues of Laplace operator scales with the order  $\lambda_i \sim i^{2/d}$. Then $\sum_{i=1}^M \lambda_i^{-d-1}$ is in the order $\sum_{i=1}^M i^{-2}$, which is bounded by $\frac{\pi^2}{6}$.
% The first term in \eqref{eqn:4} can be bounded with triangle inequality and Cauchy-Schwartz inequality as
% \begin{align}
%    & \nonumber \left|\sum_{i=1}^M \hat{h}(\lambda_i)\langle f,\bm\phi_i \rangle \bm\phi_i  - \sum_{i=1}^M \hat{h}(\lambda_{J_i})\langle f,\bm\phi_{i} \rangle \bm\phi_{i}  \right| \\
%    & \nonumber \leq \sum_{i=1}^M \left|\hat{h}(\lambda_i)-\hat{h}(\lambda_{J_i})\right|\|f\|\|\bm\phi_i\|^2.
% \end{align}
% This can be derived based on the low-pass filter assumption.

The second term in \eqref{eqn:3} can be rewritten with the definition of $J_i$ and $I$ as 
\begin{align}  
 &\nonumber \left\|\sum_{i=1}^M \hat{h}(\lambda_{J_i})\langle f,\bm\phi_{i} \rangle \bm\phi_{i}
    -  \sum_{i=1}^M \hat{h}(\lambda_i')\langle f,\bm\phi_i' \rangle \bm\phi_i'\right\|\\
&   = \Bigg\|\sum_{i=1}^M \hat{h}(\lambda_{J_i})\langle f,\bm\phi_{i} \rangle \bm\phi_{i} - \sum_{i=1}^M \hat{h}(\lambda_{J_i})\langle f,\bm\phi'_{i} \rangle \bm\phi'_{i}  +\sum_{i=1}^M \hat{h}(\lambda_{J_i})\langle f,\bm\phi'_{i} \rangle \bm\phi'_{i} - \sum_{i=1}^M \hat{h}(\lambda_i')\langle f,\bm\phi_i' \rangle \bm\phi_i'\Bigg\|\\
   &\label{eqn:4}  \leq  \Bigg\|\sum_{i=1}^M \hat{h}(\lambda_{J_i})\langle f,\bm\phi_{i} \rangle \bm\phi_{i} - \sum_{i=1}^M \hat{h}(\lambda_{J_i})\langle f,\bm\phi'_{i} \rangle \bm\phi'_{i}\Bigg\|
    + \Bigg\| \sum_{i=1}^M \hat{h}(\lambda_{J_i})\langle f,\bm\phi'_{i} \rangle \bm\phi'_{i} - \sum_{i=1}^M \hat{h}(\lambda_i')\langle f,\bm\phi_i' \rangle \bm\phi_i' \Bigg\|
  % \\
  % & \nonumber \leq \Bigg|\sum_{k=1}^K \overline{H}_k  \sum_{I_k\leq i<I_{k+1}}\langle f,\bm\phi_{i} \rangle \bm\phi_{i} \\
  % & \qquad\qquad  \qquad - \sum_{k=1}^K \underline{H}_k \sum_{I_k\leq i<I_{k+1}}  \langle f,\bm\phi_{J_i} \rangle \bm\phi_{J_i}\Bigg|\\
  % &= \Bigg|\sum_{k=1}^K \overline{H}_k  F_k  - \sum_{k=1}^K \underline{H}_k F_{I_k}\Bigg|\\
  % &\leq 2H_{max} \sum_{k=1}^K\Bigg| \langle f, \bm\Phi_k \rangle \bm\Phi_k - \langle f, \bm\Phi_{I_k} \rangle \bm\Phi_{I_k} \Bigg| \\
  % &\nonumber \leq 2H_{max} \sum_{k=1}^K\Bigg| \langle f, \bm\Phi_k \rangle  \bm\Phi_k - \langle f, \bm\Phi_k \rangle  \bm\Phi_{I_k}\Bigg| \\
  % &\qquad + 2H_{max} \sum_{k=1}^K\Bigg|\langle f, \bm\Phi_k \rangle  \bm\Phi_{I_k} - \langle f, \bm\Phi_{I_k} \rangle \bm\Phi_{I_k} \Bigg| \\
  % &\leq 2H_{max} K \|f\|
\end{align}
The first term in \eqref{eqn:4} can be bounded as 
\begin{align}
    &\nonumber \Bigg\|\sum_{i=1}^M \hat{h}(\lambda_{J_i})\langle f,\bm\phi_{i} \rangle \bm\phi_{i} - \sum_{i=1}^M \hat{h}(\lambda_{J_i})\langle f,\bm\phi'_{i} \rangle \bm\phi'_{i}\Bigg\|\\
    &  \leq \Bigg\|\sum_{k=1}^K \lambda_{I_k}^{-d} \sum_{I_k\leq i<I_{k+1}}\langle f,\bm\phi_{i} \rangle \bm\phi_{i}   - \sum_{k=1}^K \lambda_{I_{k+1}}^{-d}\sum_{I_k\leq i<I_{k+1}} \langle f,\bm\phi'_{i} \rangle \bm\phi'_{i}\Bigg\|\\
    &\leq \left\|\sum_{k=1}^K \lambda_{I_k}^{-d} F_k - \sum_{k=1}^K \lambda_{I_{k+1}}^{-d} F_k'\right\|\\
    & = \left\|\sum_{k=1}^K \lambda_{I_k}^{-d} \langle f,\bm\Phi_k \rangle \bm\Phi_k - \sum_{k=1}^K \lambda_{I_{k+1}}^{-d} \langle f,\bm\Phi'_k \rangle \bm\Phi'_k \right\|\\
    &\leq 2 \sum_{k=1}^K\lambda_{I_{k}}^{-d}\|f\|_\ccalM \|\bm\Phi_k-\bm\Phi_k'\|\\
    & \label{eqn:khan}\leq  \frac{\pi^3}{3}\frac{\gamma}{|\lambda_{M+1}-\lambda_M|}\|f\|_\ccalM,
\end{align}
where \eqref{eqn:khan} is derived based on Lemma 2 in \cite{wang2024stability}, which is Davis-Khan theorem. Similarly, as $K$ becomes large, consider $\lambda_i\sim i^{2/d}$, the summation can be bounded as $\frac{\pi^2}{6}$. We denote the eigenvalue gap as $\theta^{-1} = |\lambda_{M+1}-\lambda_M|^{-1}$, which is finite as $M$ is always finite even as the number of nodes grow.

The second term in \eqref{eqn:4} can be bounded similarly as steps in \eqref{eqn:part2-1} to \eqref{eqn:part2}.
\begin{align}
   \Bigg\| \sum_{i=1}^M \hat{h}(\lambda_{J_i})\langle f,\bm\phi'_{i} \rangle \bm\phi'_{i} - \sum_{i=1}^M \hat{h}(\lambda_i')\langle f,\bm\phi_i' \rangle \bm\phi_i' \Bigg\| \leq \left(\frac{\pi^2}{6} \beta +\frac{\pi^2}{3}\gamma\right)C_L\|f\|_\ccalM. \label{eqn:}
\end{align}
Combine \eqref{eqn:part1}, \eqref{eqn:part2} and \eqref{eqn:khan} together with \eqref{eqn:}, we can get the bound for the second term of \eqref{eqn:1}, which leads to $C_4 = C_g + \ell_{max}+\frac{\pi^2(\lambda_2+2)}{3}\sqrt{d}C_L\|f\|_\ccalM+ \frac{\pi^3}{3}\theta^{-1}\|f\|_\ccalM$.

This concludes that
\begin{align}
   GA_\tau \leq  C_1\frac{\epsilon }{\sqrt{N}}
     + C_2 \frac{\sqrt{\log(1/\delta)}}{N}   +    C_3 \left(\frac{\log N}{N}\right)^{\frac{1}{d}} + C_4 \gamma.
\end{align}

\subsection{Proof of Theorem 2}
Theorem 2 can be extended based on the proof process of Theorem 1. Except there are $K$ manifolds to consider parallelly, which leads to a summation of all the $K$ manifold generalization errors. The only difference lies in that the target function remains the same even under model mismatch, which leads to the constant $C_4$ independent of the target function.

\subsection{Extension to Out-of-Distribution Generalization}
The authors of \cite{shen2024zero} propose a novel method for mapping the features from different domains.
In \cite{ding2021closer}, the authors show that GNNs with augmentations perform better under domain shifts with pervasive empirical results. The authors in \cite{fan2023generalizing} provide a casual representational learning method for GNN out-of-distribution generalization. These methods and architectures to cope with domain shifts are summarized in \cite{li2022ood}. While only the authors in \cite{baranwal2021graph} provide a theoretical analysis on out-of-distribution generalization with a contextual stochastic block model as the graph generative model.

An important implication of generalization over deformation is to characterize the domain generalization capability. The conclusion of GNN generalization with model mismatch on the node level can extend to bound the domain generalization gap of GNNs by seeing the domain shift as the manifold model shift. 

Suppose the nodes on the training graph $\bbG_1$ are sampled from manifold $\ccalM_1$ (equipped with Laplace operator $\ccalL_1$) while the testing nodes are sampled from manifold $\ccalM_2$ (equipped with Laplace operator $\ccalL_2$). The out-of-distribution generalization error is denoted as 
\begin{align}
    GA_{OOD} = \sup_{\bbH\in\ccalH}|R_{\ccalM_2}(\bbH) - R_{\bbG_1}(\bbH)|,
\end{align}
with $R_{\ccalM_2}$ representing the statistical risk over the shifted manifold $\ccalM_2$ and $R_{\bbG_1}$ the empirical risk over the training graph $\bbG_1$ sampled from $\ccalM_1$. A similar conclusion can be obtained by measuring the operator difference between $\ccalL_1$ and $\ccalL_2$, which is used to measure the difference between models $\ccalM_1$ and $\ccalM_2$.
\begin{proposition}
    Suppose a neural network is equipped with filters defined in Definition 2 and normalized nonlinearites (Assumption 2). The neural network is operated on a graph $\bbG_1$ generated according to (6) from $\ccalM_1$ and a shifted manifold $\ccalM_2$ with a bandlimited (Definition 1) input manifold function. The out-of-distribution generalization of neural network trained on $\bbG_1$ to minimize a normalized Lipschitz loss function (Assumption 3) holds in probability at least $1-\delta$ that
    \begin{align}
   & \nonumber GA_{OOD}\leq C_1\frac{\epsilon }{\sqrt{N}}
     + C_2 \frac{\sqrt{\log(1/\delta)}}{N}   +    C_3 \left(\frac{\log N}{N}\right)^{\frac{1}{d}}   + C_4 \|\ccalL_1 -\ccalL_2\|_{op},
\end{align}
with $C_1$ and $C_3$ depending on the geometry of $\ccalM_1$, $C_4$ depending on the Lipschitz continuity of the target function, $C_1$ and $C_4$ depending on the continuity of the filter functions.
\end{proposition}
\begin{proof}
    The proof process is in accordance with the proof of Theorem 1, except that the differences between eigenvalues of $\ccalM_1$ and $\ccalM_2$ are now bounded as follows.
    \begin{lemma}
The eigenvalues of LB operators $\ccalL_1$ and $\ccalL_2$  satisfy
\begin{align}
    &|\lambda_{1,i}-\lambda_{2,i}|\leq \|\ccalL_1-\ccalL_2\|_{op}, \text{ for all }i=1,2\hdots
\end{align}
    \end{lemma}
    Combined with the Davis-Khan Theorem to bound the eigenspace difference as 
    \begin{lemma}
        Suppose the spectra of $\ccalL_1$ and $\ccalL_2$ are partitioned as $\theta \bigcup \Theta$ and $\omega \bigcup \Omega$ respectively, with $\theta \bigcap \Omega = \emptyset$ and $\omega \bigcap \Theta = \emptyset$. If there exists $\beta>0$ such that $\min_{x\in \theta, y\in \Omega}|x-y|\geq \beta$ and $\min_{x\in \omega, y\in \Theta}|x-y|\geq \beta$, then
        \begin{equation}
            \|E_{\ccalL_1}(\theta) - E_{\ccalL_2}(\omega)\| \leq \frac{\pi}{2} \frac{\|\ccalL_1 -\ccalL_2\|_{op}}{d}
        \end{equation}
    \end{lemma}
\end{proof}

\section{Further Reference}
\paragraph{Graphon theory}
Some works consider the generative model of graphs as graphons. These works covered the properties of GNNs, including convergence, stability, and transferability \cite{ruiz2020graphon, maskey2023transferability, keriven2020convergence}. Graphon can also be used as a pooling operator in GNNs \cite{parada2023graphon}. Although graphon is a straightforward model with simple calculations, it imposes several limitations compared to the manifold model that we prefer. Firstly, the graphon model assumes an infinite degree at every node of the generative graph \citep{lovasz2012large}, which is not the case in most scenarios in practice. 
Secondly, the graphon model gives little intuition about the geometry of the underlying model, which makes it hard to visualize a graphon. On the other hand, manifolds are more interpretable, as simple cases can be found on the sphere or other 3D shapes. Lastly, the manifold model is a more realistic assumption even for high-dimensional data \cite{lunga2013manifold}. 

\paragraph{Manifold Theory} 
We consider a manifold model to be a generative model of graphs. Related to us, some works have focused on manifold learning. In particular, some works have set out to test the manifold assumption of the data \cite{fefferman2016testing}, the manifold dimension \cite{farahmand2007manifold} and show the complexity of doing so \cite{narayanan2009sample,aamari2021statistical}. In terms of prediction and classification using manifolds and manifold data, several algorithms and methods have been proposed -- notable the \textit{Isomap} algorithm \cite{choi2004kernel,wu2004extended,yang2016multi} and other manifold learning techniques \cite{talwalkar2008large}. In all, these techniques utilize the manifold assumption to infer properties of the manifold, which differs from our goal of providing a generalization bound for GNNs generated from the manifold. 

\paragraph{Transferability of GNNs} Transferability of GNNs has been analyzed in many works by comparing the output differences of GNNs on different sizes of graphs when graphs converge to a limit model without statistical generalization analysis. In \cite{ruiz2023transferability,ruiz2020graphon, maskey2023transferability}, the transferability of GNNs on graphs generated from graphon models is analyzed by bounding the output differences of GNNs. In \cite{cervino2023learning}, the authors show how increasing the size of the graph as the GNN learns, generalizes to the large-scale graph. In \cite{levie2021transferability}, the authors analyze the transferability of GNNs on graphs generated from a general topological space while the authors in \cite{wang2023geometric} focus on graphs generated from manifolds. In \cite{le2024limits}, the authors propose a novel graphop operator as a limit model for both dense and sparse graphs with a transferability result proved. In \cite{lee2017transfer} and \cite{zhu2021transfer}, the authors study the transfer learning of GNNs with measurements of distance between graphs without a limit assumption. In \cite{he2023network}, a transferable graph transformer is proposed and proved with empirical results.